\documentclass{article}
\usepackage{iclr2027_conference,times}

\usepackage{amsmath,amsfonts,bm}

\def\eqref#1{equation~\ref{#1}}

\def\1{\bm{1}}

\DeclareMathAlphabet{\mathsfit}{\encodingdefault}{\sfdefault}{m}{sl}
\SetMathAlphabet{\mathsfit}{bold}{\encodingdefault}{\sfdefault}{bx}{n}

\usepackage{graphicx}
\usepackage{caption}
\usepackage{subcaption}
\usepackage{algorithm}
\usepackage{algorithmic}
\usepackage{hyperref}
\usepackage{fontawesome5}
\usepackage{tikz}
\usepackage{bbm}
\usepackage{url}
\usepackage{booktabs}
\usepackage{amsthm}
\usepackage{xcolor}
\hypersetup{
    colorlinks=true,
    citecolor=olive,
    linkcolor=teal,
    urlcolor=magenta
}
\usepackage[capitalize,nameinlink,noabbrev]{cleveref}
\definecolor{eqhighlight}{RGB}{242,246,249}
\definecolor{leapblue}{RGB}{42,92,140}
\definecolor{leaplight}{RGB}{239,245,250}
\definecolor{leapgray}{RGB}{246,247,248}
\definecolor{leaptext}{RGB}{55,75,95}

\crefname{equation}{eq.}{eqs.}
\Crefname{equation}{Eq.}{Eqs.}
\crefname{figure}{fig.}{figs.}
\Crefname{figure}{Fig.}{Figs.}
\crefname{section}{\S\!}{\S\!}
\Crefname{section}{\S\!}{\S\!}
\crefname{theorem}{thm.}{thms.}
\Crefname{theorem}{Thm.}{Thms.}
\crefname{proposition}{prop.}{props.}
\Crefname{proposition}{Prop.}{Props.}
\crefname{lemma}{lem.}{lems.}
\Crefname{lemma}{Lem.}{Lems.}
\crefname{corollary}{cor.}{cors.}
\Crefname{corollary}{Cor.}{Cors.}
\crefname{remark}{rem.}{rems.}
\Crefname{remark}{Rem.}{Rems.}
\crefname{definition}{def.}{defs.}
\Crefname{definition}{Def.}{Defs.}
\crefname{algorithm}{alg.}{algs.}
\Crefname{algorithm}{Alg.}{Algs.}
\crefname{assumption}{assump.}{assumps.}
\Crefname{assumption}{Assump.}{Assumps.}
\usetikzlibrary{arrows.meta,positioning,calc,decorations.pathreplacing}
\theoremstyle{plain}
\newtheorem{theorem}{Theorem}[section]
\newtheorem{proposition}[theorem]{Proposition}

\theoremstyle{definition}

\theoremstyle{remark}
\newtheorem{remark}[theorem]{Remark}

\title{Action Learning for Scientific Simulations}

\author{%
\textbf{Sarah Liaw}$^{1,2}$\thanks{Correspondence to \texttt{sliaw@g.harvard.edu}.}
\quad
\textbf{Yuanqi Du}$^{3}$
\quad
\textbf{Yilun Du}$^{1,2}$
\\
{\normalfont $^{1}$Kempner Institute
\quad
$^{2}$Harvard University
\quad
$^{3}$Microsoft Research New England}
\\[6pt]
\href{https://sarahliaw.github.io/action-learning/}{\textcolor{leapblue}{\faGlobeAmericas\ \textbf{Website}}}
\qquad
\href{https://github.com/SarahLiaw/action-learning}{\textcolor{leapblue}{\faGithub\ \textbf{Code}}}
}

\iclrfinalcopy
\begin{document}

\maketitle
\lhead{Preprint}

\begin{abstract}
Understanding and controlling how systems evolve is central to science and engineering, where many problems require reasoning about complete trajectories. Existing approaches typically learn local dynamics or generate trajectories, but objectives constructed from individual transitions can miss dependencies across time, while generative models do not necessarily provide a tractable objective for evaluating and optimizing arbitrary paths. We study \emph{action learning}, the problem of recovering a stochastic action from trajectory observations, yielding a common objective for reasoning about dynamical systems. We introduce \textbf{LEAP} (\textbf{L}earned \textbf{E}nergy on \textbf{A}ction \textbf{P}aths), which anchors the unknown action to a tractable reference process and learns the remaining correction by comparing observed trajectories with reference trajectories. This construction yields a differentiable objective on path space that can capture both local dynamical structure and nonadditive dependencies across time. Once learned, the same action can be flexibly reused for path ranking and inference, mechanism steering, constrained optimization and sampling. Across stochastic physical systems, many-body orbital dynamics and molecular simulations, LEAP recovers underlying dynamics, captures dependencies across time and enables transfer across inference objectives and physical conditions. We show that \emph{action learning} is a promising framework for turning trajectory observations into reusable objectives for scientific inference and control.\looseness=-1 \end{abstract}

\section{From trajectory samples to path inference}
\label{sec:introduction}
\begin{figure}[htbp]
    \centering
    \captionsetup{font=small}

    \begin{subfigure}{0.48\linewidth}
        \centering
        \includegraphics[width=\linewidth]{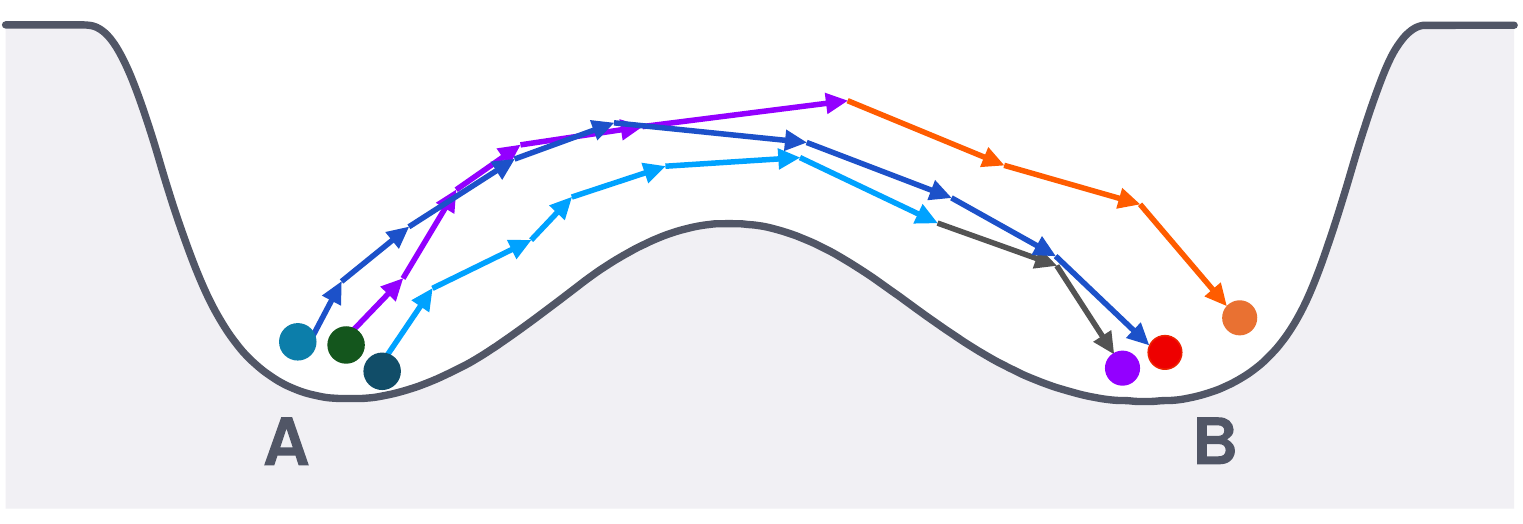}
        \caption{\textbf{Local dynamics learning.} One-step transition laws induce a likelihood over complete trajectories.}
        \label{fig:local}
    \end{subfigure}
    \hfill
    \begin{subfigure}{0.48\linewidth}
        \centering
        \includegraphics[width=\linewidth]{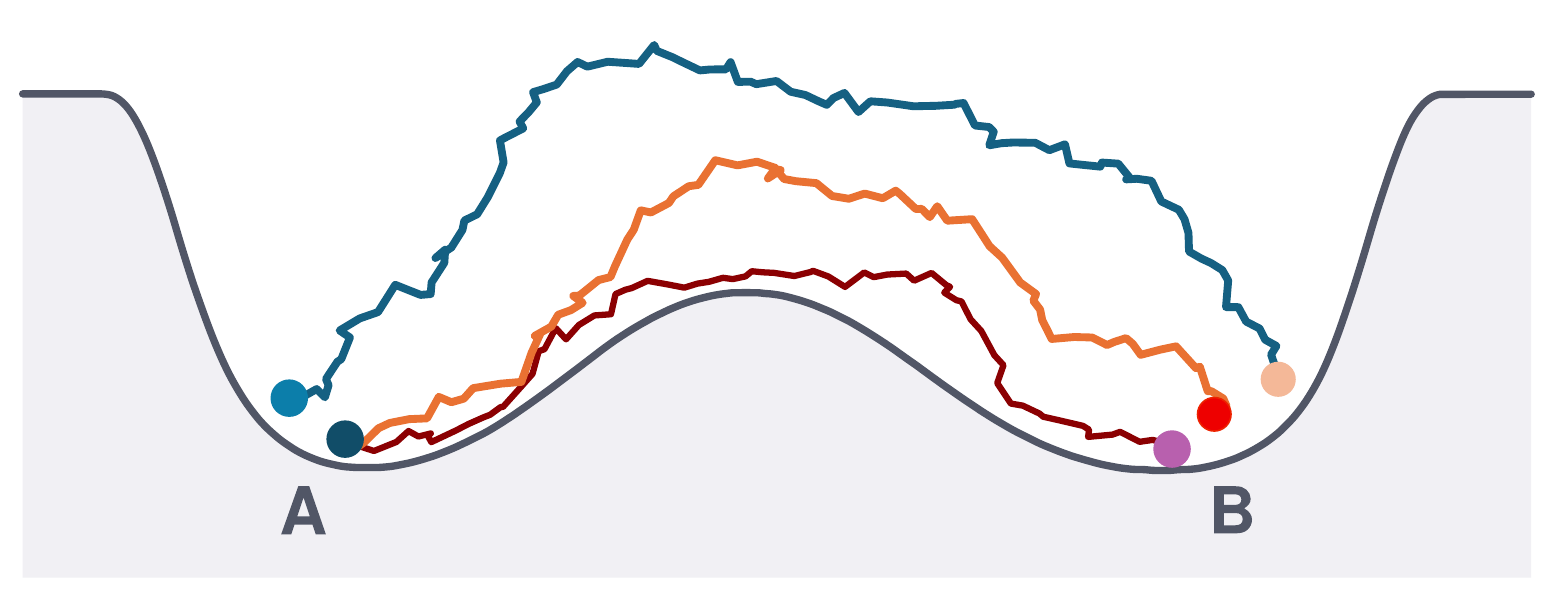}
        \caption{\textbf{Action learning.} LEAP learns relative path costs from complete trajectory comparisons.}
        \label{fig:global}
    \end{subfigure}

    \caption{\textbf{Local dynamics learning versus action learning.}
    Local models learn short-range transition laws that induce a path likelihood, whereas LEAP learns an action from complete trajectories while retaining structured local dynamics when appropriate.\\}
    \label{fig:teaser}
    \vspace{-18pt}
\end{figure}

Understanding how systems move between states requires reasoning about the trajectories that connect them. In molecular dynamics, two conformations may be connected by several transition routes that visit different intermediate configurations~\citep{bolhuis2002transition,shaw2010atom}; in robotics, agents act through sequences of decisions~\citep{kalakrishnan2011stomp,janner2022planningdiffusionflexiblebehavior}. In these settings, the object of interest is often the trajectory itself: which paths realize a transition, which mechanisms are plausible and which paths remain likely under perturbations or constraints.\looseness=-1

Answering these questions requires reasoning about complete paths. To infer representative transitions or impose new preferences, we need an objective that assigns probabilities or costs to trajectories. When the dynamics are known, stochastic action functionals provide such an objective. In statistical physics and stochastic control, actions such as the Onsager--Machlup functional, Freidlin--Wentzell rate functional and path integral control costs define probabilities or costs on path space~\citep{onsager1953fluctuations,FreidlinWentzell2012RandomPerturbations,Kappen_2005,Adib_2007}. Once the action is available, we can use it for downstream tasks like constrained inference and sampling.\looseness=-1 

In many applications, however, we do not know the action. Instead, we observe trajectories from simulations or experiments, such as molecular dynamics, sequences of state transitions, policy rollouts or logs of stochastic processes~\citep{LindorffLarsen2011HowFP,shaw2010atom,osti_10226477}. Classical trajectory inference methods typically assume a known action, while learned methods predict local dynamics, learn representations or generate trajectories directly~\citep{chen2019neuralordinarydifferentialequations,li2020scalable,song2021scorebasedgenerativemodelingstochastic,NEURIPS2024_478b06f6}. These objects are useful, but they do not directly give the reusable path objective needed for downstream trajectory inference. We therefore call the task of recovering an evaluable action from trajectory samples \emph{action learning}.\looseness=-1

Existing learned methods often model dynamics locally by fitting one-step transition laws and using the resulting local model to evaluate complete trajectories~\citep{3327345.3327385}. This requirement can become restrictive when trajectories are selected by global events, e.g., successful transitions or task-level constraints. A local predictor trained on such data estimates the one-step conditionals of the selected process, which can differ from the underlying dynamics. Errors in local transition log-likelihoods can also accumulate along a path and then distort the induced optimization landscape. Since path-level quantities do not have to decompose into independent local decisions, optimizing locally plausible segments can still produce a poor local minimum; \Cref{fig:teaser} shows this distinction.\looseness=-1

We propose \textbf{LEAP} (\textbf{L}earned \textbf{E}nergy on \textbf{A}ction \textbf{P}aths), a framework for learning a stochastic action directly from trajectory samples. LEAP treats the action itself as the learning target, rather than a local predictor or trajectory generator. It anchors the unknown action to a tractable reference process and learns the remaining correction from complete trajectories, yielding an evaluable objective on path space. In structured settings, this correction recovers changes in the underlying dynamics; more generally, it can capture nonadditive interactions across transitions. Once learned, the action can be reused with preferences that were not specified during training. For example, a reward for visiting an intermediate basin can favor a different transition route, while a geometric penalty can discourage paths outside admissible regions, in both cases without retraining the learned action.

Our approach applies broadly to trajectories from dynamical systems, including molecular simulations, stochastic physical systems and logged sequential processes~\citep{shaw2010atom, kalakrishnan2011stomp, janner2022planningdiffusionflexiblebehavior}. We evaluate LEAP across controlled physical systems, a $31$-body orbital benchmark with $186$-dimensional states and an alanine dipeptide molecular benchmarks. These experiments test physical recovery, nonadditive dependence across time and reuse across inference objectives, with the double well and orbital systems providing exact actions for evaluation. \looseness=-1

Our contributions are as follows: (i) we formulate \emph{action learning} as learning a reusable objective directly on complete trajectories; (ii) we introduce \textbf{LEAP}, which learns a stochastic action by comparing observed trajectories with samples from a tractable reference process; (iii) we establish theoretical guarantees for recovering the action and show that dependencies across time can result in trajectory objectives that cannot be represented by sums of local transition costs; and (iv) we demonstrate recovery of underlying physical dynamics from terminally selected trajectories, modeling of nonadditive dependencies in many-body systems, and reuse of a fixed learned action across molecular inference tasks and physical conditions.
\subsection{Related work}
\label{sec:related_work}

\textbf{Sampling trajectories from known dynamics.} Classical methods for sampling trajectories usually assume access to the dynamics, energy landscape or action. When trajectories of interest occur often, it is sufficient to just do direct simulation. When relevant paths are rare, e.g., because an energy barrier separates metastable states, methods can use known physical models to identify or sample transition pathways. The string method evolves discretized curves on an energy landscape~\citep{PhysRevB.66.052301}, while minimum action methods optimize action functionals derived from stochastic dynamics~\citep{FreidlinWentzell2012RandomPerturbations}. Stochastic control methods construct trajectory objectives from known dynamics and costs to sample rare transitions~\citep{Kappen_2005,Todorov2009EfficientComputationOptimalActions,WilliamsAldrichTheodorou2017MPPI,holdijk2023stochastic,du2024doob}. More recent methods use learned force fields or pretrained generative models to get energies or dynamics for transition inference~\citep{yuan2024analytical,raja2025actionminimizationmeetsgenerativemodeling,amin2025fastspecializedmachinelearning}. These methods typically rely on available physical models, learned dynamics or predefined objective to perform inference. LEAP targets the complementary setting of learning the trajectory action itself from observations, enabling subsequent inference without access to the underlying target dynamics at all.

\textbf{Learning from trajectories.}
Learned trajectory models usually target evolution or generation. Neural differential equation and latent dynamics models learn evolution from observations~\citep{chen2019neuralordinarydifferentialequations,kidger2021neural,pmlr-v97-hafner19a}, while score-based, flow matching, diffusion and action matching methods learn vector fields or scores that transport or generate samples from state, trajectory or marginal distributions~\citep{song2021scorebasedgenerativemodelingstochastic,NEURIPS2024_478b06f6,jing2024alphafoldmeetsflowmatching,neklyudov2023actionmatchinglearningstochastic}. These methods model evolution or generation, but they do not generally return the stochastic action needed for path optimization, mechanism steering or Metropolis-Hastings correction. Energy-based models have also been used for planning by
learning energies over individual state transitions and composing them along a trajectory~\citep{pmlr-v100-du20a}. Relatedly, reinforcement learning and imitation learning learn rewards or policies from demonstrations, e.g., maximum-entropy inverse reinforcement learning and adversarial imitation learning~\citep{ziebart2008maxentirl,ho2016generativeadversarialimitationlearning,fu2018learningrobustrewardsadversarial,nuti2023extractingrewardfunctionsdiffusion}. LEAP also learns from trajectories, but targets a different object: a stochastic action on complete paths, i.e., the trajectory-level objective used directly for optimization, control and sampling, which can then be reused across new downstream objectives. \looseness=-1

\section{Learning action functionals from trajectory samples}
\vspace{-3pt}
\label{sec:method}
LEAP learns a differentiable stochastic action $S_\theta(\Phi;c)$ directly from trajectory samples, providing a reusable objective for evaluating and optimizing trajectories. The action assigns a scalar objective to each complete trajectory, with lower action corresponding to higher probability under the observed path distribution. LEAP requires trajectory observations but not the underlying drift, potential or simulator likelihood. We construct the action by introducing a tractable reference process and learning a correction that describes how the observed path distribution differs from the reference. This formulation can represent both local dynamical structure and dependencies across time, yielding an objective that can be evaluated and reused for new inference tasks. See \Cref{alg:leap_train} for the algorithm. \looseness=-1

\begin{figure}[t]
    \centering
    \captionsetup{font=small}

    \includegraphics[width=\linewidth]{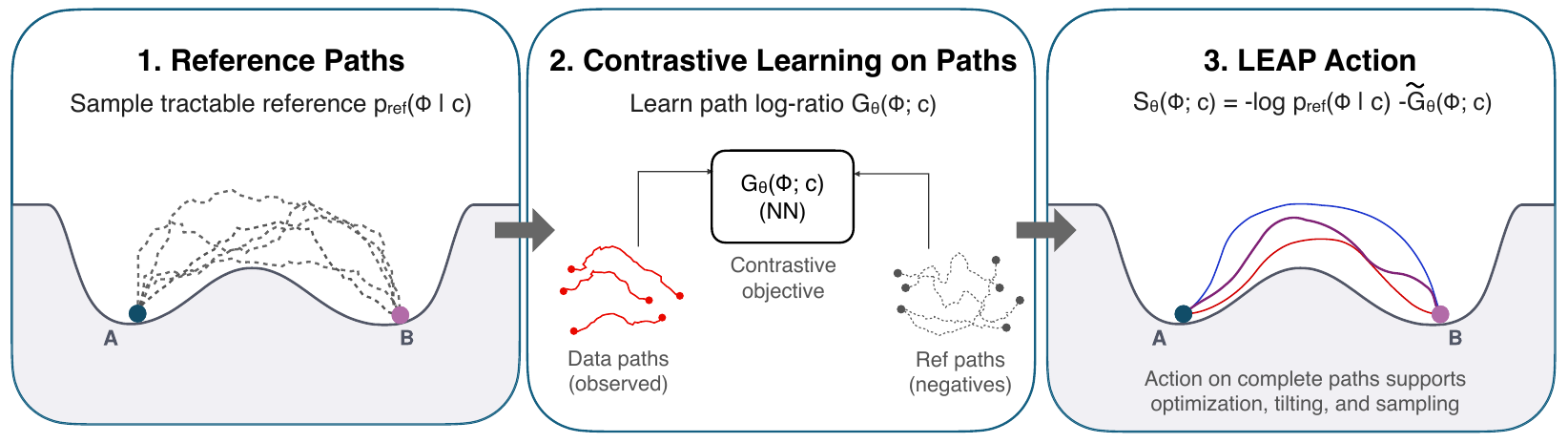}

    \caption{\textbf{LEAP.} Given data trajectories and a tractable reference process, LEAP learns a centered path-space correction and combines it with the reference action to obtain $S_\theta(\Phi;c)$.}
    \label{fig:method_overview}
    \vspace{-15pt}
\end{figure}

\subsection{Target action and training objective} \label{sec:method_setup}

Let $\Phi=(x_0,x_1,\dots,x_N)$ denote a discretized trajectory with $t_k=k\Delta t$ and horizon $T=N\Delta t$. Let $c$ denote the conditioning variables, e.g., endpoints, terminal regions, horizon or diffusion level. We observe trajectory samples $\Phi\sim p_{\mathrm{data}}(\cdot\mid c)$ for each context.

\textbf{Target action.} Following stochastic action formalisms such as Onsager--Machlup and Freidlin--Wentzell theory, we use the negative log path density as the trajectory objective~\citep{onsager1953fluctuations, FreidlinWentzell2012RandomPerturbations}: \begin{equation} \label{eq:target_action} S^*(\Phi;c) = -\log p_{\mathrm{data}}(\Phi\mid c) + \alpha(c), \end{equation} where $\alpha(c)$ is constant across paths with the same context and therefore does not affect their relative action, including for path ranking, sampling or representative path optimization (\Cref{sec:method_downstream}).

\textbf{Reference-based decomposition.} Directly evaluating $p_{\mathrm{data}}(\Phi\mid c)$ is generally not possible from trajectory samples. LEAP therefore introduces a tractable reference process with density $p_{\mathrm{ref}}(\Phi\mid c)$, providing an evaluable baseline action against which the observed dynamics can be characterized.  The unknown action can be expressed as the reference action plus a correction that captures how the observed path distribution differs from the reference. This correction admits a physical interpretation when the processes differ in drift, while more generally allowing dependencies across transitions that cannot be represented by additive local objectives.

Specifically, the correction is the path log-density ratio \begin{equation} \label{eq:path_ratio_target} G^*(\Phi;c) = \log \frac{p_{\mathrm{data}}(\Phi\mid c)} {p_{\mathrm{ref}}(\Phi\mid c)}. \end{equation} With $S_{\mathrm{ref}}(\Phi;c)=-\log p_{\mathrm{ref}}(\Phi\mid c)$, the target action is \begin{equation} \label{eq:action_decomposition} S^*(\Phi;c) = \underbrace{S_{\mathrm{ref}}(\Phi;c)}_{\text{reference action}} - \underbrace{G^*(\Phi;c)}_{\text{learned correction}} + \alpha(c). \end{equation} We estimate $G^*$ from comparisons between observed and reference trajectories and combine it with the reference action. The resulting $S_\theta$ is the learned stochastic action used for trajectory inference.

\textbf{Learning the action correction.} To estimate $G^*$, we compare observed trajectories with those drawn from the reference process~\citep{gutmann2010noise}. For each context $c$, we draw one observed trajectory $\Phi^+\sim p_{\mathrm{data}}(\cdot\mid c)$ and $K$ reference trajectories $\Phi_k^-\sim p_{\mathrm{ref}}(\cdot\mid c)$. We parameterize the correction by $G_\theta(\Phi;c)$ and optimize the setwise comparison objective, which has the InfoNCE form~\citep{oord2019representationlearningcontrastivepredictive}

 \begin{equation} \label{eq:cl_main} \mathcal L_K(\theta) = -\mathbb E \left[ \log \frac{ \exp G_\theta(\Phi^+;c) }{ \exp G_\theta(\Phi^+;c) + \sum_{k=1}^{K} \exp G_\theta(\Phi_k^-;c) } \right]. \end{equation}

This objective has the desired correction as its population solution. Under the conditions in \Cref{app:bayes_ratio}, every population minimizer satisfies \begin{equation} \label{eq:ratio_recovery} G_{\theta^*}(\Phi;c) = G^*(\Phi;c)+a(c), \end{equation} for every $K\geq1$. The context-dependent term cancels from each comparison, so the objective identifies the part of the correction that varies across trajectories.

Because the comparison objective leaves this context-dependent constant undetermined, we fix the additive convention after training by centering the correction under the reference distribution, \begin{equation} \label{eq:centered_ratio} \bar G_\theta(c) = \mathbb E_{\Phi\sim p_{\mathrm{ref}}(\cdot\mid c)} \left[G_\theta(\Phi;c)\right], \qquad \widetilde G_\theta(\Phi;c) = G_\theta(\Phi;c)-\bar G_\theta(c). \end{equation} The learned action is then \begin{equation} \label{eq:anchored_energy} S_\theta(\Phi;c) = -\log p_{\mathrm{ref}}(\Phi\mid c) - \widetilde G_\theta(\Phi;c). \end{equation} At the population solution, this recovers the data action up to a constant depending on the context, which does not affect path rankings or optimization within that context.

\subsection{Reference process}
\label{sec:method_reference}

The reference defines the baseline path measure and the inductive bias of the learned action. This means we need a tractable path density, efficient sampling and sufficient probability mass over the trajectories represented in the data. While its drift can include an available nominal drift or force model, LEAP does not necessarily require one. For rare-event trajectories, we condition the reference on the same event as the data so that both contain trajectories selected by the same condition.

We use a discrete-time Gaussian Markov reference because its transition law is analytic and parameterized directly by drift and diffusion. Under shared diffusion, differences in drift induce an exact path log-likelihood ratio, giving the learned correction a direct dynamical interpretation. We define\looseness=-1
\begin{equation}
\label{eq:ref_dynamics}
x_{k+1}
=
x_k
+
b_{\mathrm{ref}}(x_k,t_k;c)\Delta t
+
\sqrt{2D\Delta t}\,\xi_k,
\qquad
\xi_k\sim\mathcal N(0,I),
\end{equation}
where $b_{\mathrm{ref}}$ is the reference drift and $D$ is the diffusion coefficient. This induces the path density $p_{\mathrm{ref}}(\Phi\mid c)$ and reference action $S_{\mathrm{ref}}(\Phi;c)$. For endpoint or terminal-region conditioning, $p_{\mathrm{ref}}$ denotes the corresponding conditional density over the interior knots; its normalizer contributes only a context-dependent constant (see \Cref{app:matched_conditioning}). \looseness=-1

\subsection{Local and global action corrections}
\label{sec:method_parameterization}

When the data and reference processes are Gaussian Markov processes with the same diffusion coefficient and differ only in drift, their exact path log-likelihood ratio decomposes over transitions. We use this structure when the drift difference is the quantity of interest. Let $v_k=(x_{k+1}-x_k)/\Delta t$ and let $u_\theta$ parameterize $b_{\mathrm{data}}-b_{\mathrm{ref}}$. The correction is \looseness=-1
\begin{equation}
\label{eq:ratio_param}
G_\theta(\Phi;c)
=
\sum_{k=0}^{N-1}
\left[
\underbrace{
\frac{1}{2D}\,
u_\theta(x_k,t_k;c)^\top
\bigl(v_k-b_{\mathrm{ref}}(x_k,t_k;c)\bigr)
}_{\text{velocity alignment}}
-
\underbrace{
\frac{1}{4D}\,
\|u_\theta(x_k,t_k;c)\|^2
}_{\text{normalization cost}}
\right]\Delta t.
\end{equation}
When the data and reference processes also share the initial law and time grid, and differ only through their drift, \Cref{eq:ratio_param} is the exact path log-likelihood ratio for $b_{\mathrm{data}}=b_{\mathrm{ref}}+u^*$; see \Cref{app:em_ratio}. For stopped or padded trajectories, we sum only over active transitions as described in the appendix.

More generally, LEAP does not require a Gaussian drift parameterization. When the log density ratio between the data and reference distributions decomposes across transitions, LEAP-Local learns a scalar correction $g_\theta$ for each transition and sums these corrections along the trajectory,\begin{equation} \label{eq:local_correction} 
\vspace{-5pt}G_\theta^{\mathrm{local}}(\Phi;c) = \sum_{k=0}^{N-1} g_\theta(x_k,x_{k+1},t_k;c). 
\end{equation}
This factorization can fail when information is shared across the trajectory. For example, marginalizing a latent variable that remains fixed throughout a trajectory can couple multiple transitions. LEAP-Global therefore models the correction directly from the
complete path,
\begin{equation} \label{eq:global_correction} G_\theta^{\mathrm{global}}(\Phi;c) = r_\theta(\Phi;c). \end{equation}

LEAP-Hybrid combines the two parameterizations, \begin{equation} \label{eq:hybrid_correction} G_\theta^{\mathrm{hybrid}}(\Phi;c) = G_\theta^{\mathrm{local}}(\Phi;c) + G_\theta^{\mathrm{global}}(\Phi;c). \end{equation} The global term captures interactions across transitions that cannot be written as an additive transition correction. When both terms are learned jointly, only their sum is identified, so interpreting the local term as a physical drift correction requires additional restrictions.

For conservative systems, we optionally parameterize the total drift as the gradient of a potential,
\begin{equation}
\label{eq:conservative}
b_{\mathrm{ref}}(x,t;c)+u_\theta(x,t;c)
=
-\nabla U_\theta(x;c).
\end{equation}

\subsection{Downstream trajectory inference}
\label{sec:method_downstream}

Because $S_\theta$ is an action on complete trajectories, we can reuse it after training. The same learned object can rank candidate paths, infer representative trajectories, incorporate mechanism preferences, impose soft constraints and define trajectory distributions. See \Cref{app:inference} for details.\looseness=-1

\begin{figure}[t]
    \centering
    \captionsetup{font=small}

    \includegraphics[width=\linewidth]{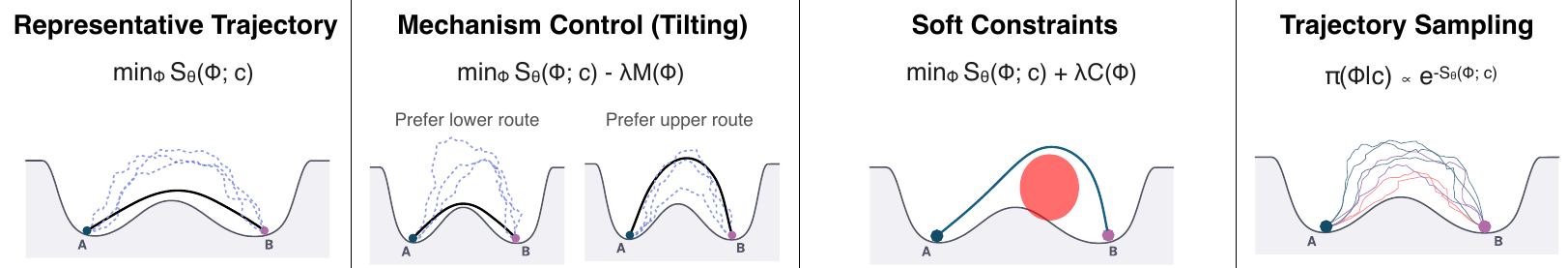}

    \caption{\textbf{Downstream trajectory inference with the learned action.}
    Once $S_\theta(\Phi;c)$ is learned, the same action supports representative path inference, mechanism steering (tilting), constrained inference and sampling.}
    \label{fig:downstream_tasks}
    \vspace{-15pt}
\end{figure}

\textbf{Trajectory ranking and representative inference.}
Given feasible trajectories $\Gamma(c)$, e.g., paths satisfying boundary conditions, we can rank candidate trajectories by their action or infer the lowest-action representative trajectory by solving
\begin{equation}
\label{eq:map_problem}
\Phi^*(c)
=
\arg\min_{\Phi \in \Gamma(c)}
S_\theta(\Phi;c).
\end{equation}

\textbf{Mechanism steering and constraints.}
We steer the representative path by adding trajectory-level functionals. Given a desired mechanism preference $M(\Phi)$ and a constraint penalty $C(\Phi)$, we solve
\begin{equation}
\label{eq:traj_modification}
\arg\min_{\Phi \in \Gamma(c)}
\Bigl[
S_\theta(\Phi;c)
-\lambda_M M(\Phi)
+\lambda_C C(\Phi)
\Bigr].
\end{equation}
Changing $\lambda_M$ or $\lambda_C$ changes the inference problem without retraining the model.\looseness=-1

\textbf{Compositional trajectory distributions.}
The learned action defines an unnormalized trajectory distribution. Given mechanism preferences $M_j$ with weights $\gamma_j$ and penalties $C_\ell$ with weights $\lambda_\ell$,
\begin{equation}
\label{eq:composed_path_density}
\pi(\Phi \mid c)
\propto
\exp\!\left(
-S_\theta(\Phi;c)
+\sum_j \gamma_j M_j(\Phi)
-\sum_\ell \lambda_\ell C_\ell(\Phi)
\right).
\end{equation}
Sampling from $\pi$ gives a family of plausible paths weighted by the specified preferences/penalties.\looseness=-1

\textbf{Metropolis-Hastings (MH) sampling.}
At fixed context, $S_\theta$ defines an unnormalized distribution over paths. For a proposal kernel $q(\Phi'\mid\Phi)$ on $\Gamma(c)$, the MH acceptance probability is
\begin{equation}
\label{eq:mh_accept_general}
A(\Phi,\Phi')
=
\min\left\{
1,\,
\exp\bigl(-S_\theta(\Phi';c)+S_\theta(\Phi;c)\bigr)
\frac{q(\Phi\mid\Phi')}{q(\Phi'\mid\Phi)}
\right\}.
\end{equation}
For symmetric proposals, this reduces to $A(\Phi,\Phi')
=
\min\left\{
1,\,
\exp\bigl(-S_\theta(\Phi';c)+S_\theta(\Phi;c)\bigr)
\right\}.$
Any context-dependent additive constant in $S_\theta$ cancels in the acceptance ratio.\looseness=-1

\begin{figure}[t]
\begin{minipage}{\linewidth}
\begin{algorithm}[H]
\caption{LEAP training}
\label{alg:leap_train}
\begin{algorithmic}[1]
\REQUIRE dataset $\mathcal D=\{(\Phi_i^+,c_i)\}$, reference process
$p_{\mathrm{ref}}(\Phi\mid c)$, reference count $K$
\WHILE{not converged}
    \STATE sample minibatch $\{(\Phi_i^+,c_i)\}_{i=1}^B \subset \mathcal D$
    \FOR{$i=1$ to $B$}
        \STATE sample $\{\Phi_{i,k}^-\}_{k=1}^K
        \sim p_{\mathrm{ref}}(\cdot\mid c_i)$
        \STATE evaluate
        $G_i^+\leftarrow G_\theta(\Phi_i^+;c_i)$ and
        $G_{i,k}^-\leftarrow G_\theta(\Phi_{i,k}^-;c_i)$
    \ENDFOR
    \STATE update $\theta$ by a gradient step on the minibatch
    estimate of \Cref{eq:cl_main}
\ENDWHILE
\RETURN learned action $S_\theta$ in \Cref{eq:anchored_energy}
\end{algorithmic}
\end{algorithm}
\end{minipage}
\vspace{-15pt}
\end{figure}

\section{Why LEAP recovers an action} \label{sec:theory}

We give three guarantees for action learning. The first recovers the data action from the comparison objective, the second identifies the drift correction under the structured model, and the third bounds downstream inference error from action error. Full assumptions and proofs are in~\Cref{app:theory}.

\textbf{The comparison objective recovers the action.} Fix a context $c$. Assume that the data and reference distributions have positive densities on the same support, the candidates are conditionally independent and the score class contains $G^*$. Every population minimizer of \Cref{eq:cl_main} satisfies \begin{equation} \label{eq:main_bayes_ratio} G_{\theta^*}(\Phi;c) = \log \frac{p_{\mathrm{data}}(\Phi\mid c)} {p_{\mathrm{ref}}(\Phi\mid c)} + a(c). \end{equation} If $\mathbb E_{p_{\mathrm{ref}}}|G^*|<\infty$, reference centering then gives \begin{equation} \label{eq:main_data_action} S_{\theta^*}(\Phi;c) = S^*(\Phi;c)+\beta(c). \end{equation} The context-dependent constant keeps optimization solutions and MH ratios the same. \looseness=-1

\textbf{The structured correction recovers drift differences.} For Euler--Maruyama dynamics with the same diffusion, time step and initial distribution, changing the drift changes the path log-likelihood by the structured correction in \Cref{eq:ratio_param}. Matching the unconditioned path distribution identifies the drift at states and times visited by the process. If data and reference trajectories are selected using the terminal event, this changes the path log-likelihood ratio by condition-dependent constants. Recovering the unconditioned drift from selected trajectories needs variation across terminal conditions and the structured drift parameterization. See \Cref{sec:doublewell} for the recovery; \Cref{app:gaussian,app:matched_conditioning} for derivations. \looseness=-1

\textbf{Action error controls later inference.} Fix a context $c$ and an inference domain $\Gamma(c)$. Suppose \[ \sup_{\Phi\in\Gamma(c)} \left| S_\theta(\Phi;c)-S^*(\Phi;c)-\beta(c) \right| \leq \varepsilon. \] For the same added functional $H$, let $F_\theta=S_\theta+H$ and $F^*=S^*+H$. If $\widehat\Phi$ is $\delta$-suboptimal for $F_\theta$, \begin{equation} \label{eq:main_reuse_bound} F^*(\widehat\Phi) - \inf_{\Phi\in\Gamma(c)}F^*(\Phi) \leq 2\varepsilon+\delta. \end{equation} If $\pi_\theta\propto e^{-F_\theta}$ and $\pi^*\propto e^{-F^*}$ have finite, positive normalizing constants, then \begin{equation} \label{eq:main_distribution_bound} \left| \log \frac{\pi_\theta(\Phi\mid c)} {\pi^*(\Phi\mid c)} \right| \leq 2\varepsilon. \end{equation} Thus the same action error bounds both the objective value of an inferred trajectory and the change in its normalized trajectory density. \looseness=-1

\vspace{-1pt}
\section{Experiments} \label{sec:experiments}
\vspace{-5pt}
Our experiments evaluate four properties of action learning. We show that LEAP can recover physical dynamics from selected trajectories, capture dependencies across multiple transitions, improve as the reference moves closer to the data and reuse the learned action under new physical conditions and inference requests. See~\Cref{app:protocols} for the full experimental details.\looseness=-1

\vspace{-8pt}

\subsection{Recovering dynamics under terminal selection}
\label{sec:doublewell}
\vspace{-5pt}
 We first show that LEAP recovers the underlying drift when training trajectories are selected by their final state. We use the overdamped diffusion $dX_t=b(X_t)\,dt+\sqrt{2D}\,dW_t$ with $b(x)=x-x^3$ and $U(x)=(x^2-1)^2/4$. In this double well setting, trajectories start at $x_0=-1$ and are kept when $x_T\in[0.9,1.1]$ at $T=3$. We use $5{,}000$ selected trajectories per seed with $D=0.20$ and $\Delta t=0.01$, and a Brownian reference conditioned on the same terminal window. Conditioning on the final state changes the distribution of one-step increments, so local regression does not directly estimate the original drift. LEAP instead compares complete selected paths with the matched reference and learns the drift through the structured action correction. \looseness=-1

\begin{figure*}[t]
    \centering
    \captionsetup{font=small}

    \begin{minipage}[t]{0.36\textwidth}
        \centering
        \vspace{0pt}

        \includegraphics[
            width=\linewidth
        ]{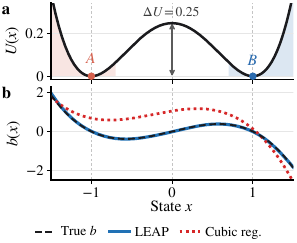}%

        \begin{subfigure}{0pt}
            \phantomsubcaption
            \label{fig:dw-potential}
        \end{subfigure}%
        \begin{subfigure}{0pt}
            \phantomsubcaption
            \label{fig:drift-reconstruction}
        \end{subfigure}

        \vspace{-14pt}
        \captionof{figure}{\textbf{Double well recovery.}}
        \label{fig:dw-recovery}
    \end{minipage}
    \hfill
    \begin{minipage}[t]{0.61\textwidth}
        \centering
        \vspace{0pt}
        \small
        \setlength{\tabcolsep}{5pt}

        \begin{tabular*}{\linewidth}{
            @{\extracolsep{\fill}}lccc@{}
        }
        \toprule
        & Drift MSE $\downarrow$
        & \multicolumn{2}{c}{True path cost $\downarrow$} \\
        \cmidrule(l){3-4}
        Method & & $T=4$ & $T=16$ \\
        \midrule
        Kramers-Moyal
        & $0.813$
        & --- & --- \\
        Cubic velocity reg.
        & $0.775$
        & $0.319$
        & $0.613$ \\
        Zero field
        & $0.0759$
        & $0.264$
        & $0.445$ \\
        Geometric path
        & ---
        & $0.264$
        & $0.264$ \\
Analytic drift
        & $0$
        & $0.253$
        & $0.250$ \\
        \midrule
        \textbf{LEAP}
        & $\mathbf{3.3\times10^{-4}}$
        & $\mathbf{0.253}$
        & $\mathbf{0.251}$ \\
        \bottomrule
        \end{tabular*}
        \vspace{-2pt}
        \captionof{table}{
            \textbf{Physical recovery and path inference.}
            Drift error and path cost under new endpoints and horizons.
        }
        \label{tab:dw-recovery}
    \end{minipage}

    \vspace{-13pt}
\end{figure*}

\textbf{Drift recovery.} We compare LEAP with Kramers-Moyal, which estimates drift from binned one-step increments, and cubic velocity regression using $\{1,x,x^2,x^3\}$, which contains the exact drift $x-x^3$. LEAP achieves drift MSE $(3.3\pm2.7)\times10^{-4}$, compared to $0.8132\pm0.0030$ for Kramers-Moyal, $0.7752\pm0.0043$ for cubic regression and $0.0759$ for the zero field (\Cref{fig:dw-recovery,tab:dw-recovery}). Even though cubic regression can represent the exact drift, it fits the one-step dynamics of the terminally selected trajectories, which follow a Doob $h$-transformed drift that differs from the original drift (\Cref{app:matched_conditioning}). \looseness=-1

\textbf{Path inference under new endpoints.} Next, we keep the learned potential fixed and change the endpoint from the training region to the saddle $x_T=0$, and vary the horizon. At $T=4$, LEAP gives true path cost $0.2533\pm0.0001$ versus the analytic value $0.2532$. At $T=16$, LEAP gives $0.2506\pm0.0006$ versus $0.2500$, while cubic velocity regression gives $0.6126\pm0.0079$ and the geometric path gives $0.2637$. The same learned action can still infer near-optimal paths after changing both the endpoint and the horizon. \looseness=-1

\textbf{Reference choice.} We retrain LEAP with different references on an unconditioned double well, where the exact path density is available. Changing the reference drift from $0$ to $0.5b$ reduces the correction error $\mathrm{NRMSE}_{\Delta G}$ from $0.66$ to $0.223$, where $\mathrm{NRMSE}_{\Delta G}$ compares predicted and exact correction differences within each context and a constant prediction gives $1$ (\Cref{app:metrics}), showing that a closer reference leaves a smaller correction to learn. See \Cref{app:reference_studies} for the full sweep.

\vspace{-7pt}
\subsection{Recovering a nonadditive trajectory action} \label{sec:orbital}
\vspace{-7pt}
Next, we test whether LEAP can learn a correction that depends on several transitions. 
These settings may appear when an unobserved physical property is fixed during a trajectory, like an unknown mass in orbital dynamics, friction coefficient in Langevin dynamics or material parameter in a mechanical system, and its effect accumulates over the full path. We model this by sampling $z\in\{-1,+1\}$ once and changing the solar mass to $m_\odot(1+z\eta)$. The same unobserved $z$ is fixed for the entire trajectory, so recovering the observed path likelihood needs combining information across time. We use a $31$-body orbital system initialized from JPL HORIZONS, giving a $186$-dimensional state space~\citep{1997BAAS...29.1099G}. \looseness=-1

Conditional on $z$, the path density factorizes over transitions. After marginalizing $z$, the exact data-to-reference correction becomes
\vspace{-3pt}
\begin{equation}
\label{eq:orbital_exact_mixture}
G^*(\Phi;c)=\log\left[
\tfrac12 e^{\ell_{+1}(\Phi;c)}+\tfrac12 e^{\ell_{-1}(\Phi;c)}
\right]-\ell_0(\Phi;c),
\end{equation}
where $\ell_z=\sum_k \ell_{z,k}$ and $\ell_0$ uses the nominal solar mass. Defining $d_k=(\ell_{+1,k}-\ell_{-1,k})/2$ and $m_k=(\ell_{+1,k}+\ell_{-1,k})/2-\ell_{0,k}$ gives \vspace{-6pt} \begin{equation} \label{eq:orbital_memory_result} G^*(\Phi;c) = \sum_k m_k + \log\cosh A, \qquad A=\sum_k d_k. \vspace{-6pt} \end{equation} Its mixed derivatives satisfy $\partial^2G^*/(\partial d_i\partial d_j)>0$ for $i\neq j$, whereas a sum of separate transition corrections has zero mixed derivatives. Therefore no additive transition model can represent the exact correction; see \Cref{app:memory_theory}.

\textbf{Local and complete-path models.} We evaluate whether the correction can be recovered from individual transitions or needs information from the whole trajectory. LEAP-Local uses the local correction from \Cref{eq:local_correction}. As a baseline with different supervision, Transition NCE applies noise-contrastive estimation~\citep{gutmann2010noise} directly to data and reference transitions using the same transition inputs. LEAP-Global uses the full trajectory and features accumulated from the nominal gravitational dynamics and its sensitivity to the solar mass (\Cref{tab:orbital_structure}). These inputs do not have the hidden variable $z$ or the exact correction $G^*$. \looseness=-1

LEAP-Local and Transition NCE stay at the constant prediction level, with NRMSE $1.0013$ and $1.0008$, while LEAP-Global reaches $0.228$, i.e., it accounts for about $95\%$ of the squared variation in the exact correction differences. This shows that the correction cannot be recovered from separate transitions and requires information across the trajectory. Consistent with this result, \Cref{fig:orbital-correction} shows that LEAP-Global follows the exact $\log\cosh A$ dependence and recovers the hidden mass perturbation as $1.02\times10^{-4}$, compared to the true $1.04\times10^{-4}$. As a trajectory diffusion comparison, GeoTDM~\citep{han2024geometric} is trained separately on the same data and reference paths, and the difference in its objectives gives AUROC $0.4999$; see \Cref{app:diffusion}.

\begin{figure*}[t]
    \centering
    \captionsetup{font=small}

    \begin{minipage}[t]{0.42\textwidth}
        \centering
        \vspace{0pt}

        \includegraphics[
            width=\linewidth,
            keepaspectratio
        ]{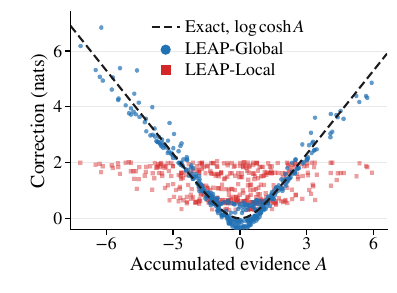}

        \vspace{-6pt}
       \captionof{figure}{
\textbf{Recovery of the nonadditive orbital correction.}
LEAP-Global follows the exact $\log\cosh A$ dependence across held-out trajectories.
}
        \label{fig:orbital-correction}
    \end{minipage}
    \hfill
    \begin{minipage}[t]{0.56\textwidth}
        \centering
        \vspace{0pt}
        \scriptsize
        \setlength{\tabcolsep}{3pt}
        \begin{tabular*}{\linewidth}{@{\extracolsep{\fill}}llc@{}}
        \toprule
        Model & Input & $\mathrm{NRMSE}_{\Delta G}\downarrow$ \\
        \midrule
        LEAP-Local
& One transition
& $1.0013{\scriptscriptstyle \pm 0.0015}$ \\
Transition NCE
& One transition
& $1.0008{\scriptscriptstyle \pm 0.0013}$ \\
\textbf{LEAP-Global}
& \textbf{Complete path + features}
& $\mathbf{0.228}{\scriptscriptstyle \mathbf{\pm 0.011}}$ \\
        \bottomrule
        \end{tabular*}
        \vspace{-2pt}
        \captionof{table}{
\textbf{Ablation on trajectory information.}
Models using one transition at a time remain near the constant prediction
level, while the complete-path model reduces NRMSE.
}
\label{tab:orbital_structure}
        \vspace{8pt}
        \setlength{\tabcolsep}{3pt}
        \begin{tabular*}{\linewidth}{@{\extracolsep{\fill}}llc@{}}
        \toprule
        Objective & Input & $\mathrm{NRMSE}_{\Delta G}\downarrow$ \\
        \midrule
Categorical NRE (LEAP)
& Transformer + features
& $0.228{\scriptscriptstyle \pm 0.015}$ \\
        Binary NRE
& Path encoder
& $1.00000{\scriptscriptstyle \pm 3\times 10^{-5}}$ \\
\textbf{Binary NRE}
& \textbf{Transformer + features}
& $\mathbf{0.175}{\scriptscriptstyle \mathbf{\pm 0.011}}$ \\
        \bottomrule
        \end{tabular*}
        \vspace{-2pt}
      \captionof{table}{
\textbf{Ablation on accumulated features.}
Adding the accumulated features reduces error for both binary and categorical
ratio objectives.
}
\label{tab:orbital_features}

    \end{minipage}

    \vspace{-15pt}
\end{figure*}

\textbf{Accumulated features.} We ablate these features while keeping a complete-path model, using binary and categorical neural ratio estimation (NRE) objectives~\citep{miller2024contrastiveneuralratioestimation}. Without them, binary NRE gives NRMSE $1.00000\pm3\times10^{-5}$. Adding them reduces error to $0.175\pm0.011$ with binary NRE and $0.228\pm0.015$ with categorical NRE (\Cref{tab:orbital_features}). Both objectives improve, showing that the accumulated information from \Cref{eq:orbital_memory_result} drives the decrease. Because both are ratio estimators on complete paths, the binary result also shows that action learning is not tied exactly to the categorical objective in \Cref{eq:cl_main}. \looseness=-1

\textbf{Scaling.} At $42$ bodies and $256$ transitions, each path has $64{,}764$ coordinates. LEAP-Global reaches NRMSE $0.111$, while LEAP-Local stays at $1.0015$. Thus, even as the system grows to more bodies and longer trajectories, LEAP-Global can accumulate the small effect of the hidden mass perturbation across transitions, while models that treat each transition separately still stay near the constant prediction baseline. See \Cref{app:orbital_scaling} for the full sweep.

\subsection{Reusing learned actions across conditions and objectives} \label{sec:reuse_results}

In molecular simulation, trajectories may come from a high-dimensional simulator but be analyzed in a small set of collective coordinates, i.e., the effective dynamics are not available in closed form. These coordinates are used to identify metastable states and describe transition pathways. They are also used in enhanced sampling and restrained simulations, where one may require a path to pass through a particular basin or compare conformational pathways across temperatures. We therefore evaluate whether LEAP can learn one action in these reduced coordinates and reuse it across these different settings. We use alanine dipeptide represented by the backbone angles $(\phi,\psi)$ on the Ramachandran torus~\citep{vargas2002alanine}. OpenMM Langevin simulations provide $169$ reactive C7eq$\to$C7ax trajectories at $500$\,K and $613$ at $600$\,K~\citep{eastman2023openmm8moleculardynamics}. \looseness=-1

\textbf{Completion from partial observations.} We keep the learned action fixed and optimize the missing states. When the endpoints and three interior states are observed, LEAP gives circular MSE $0.0931\pm0.0016$ at $600$\,K, compared to $0.119\pm0.012$ for conditional DDPM~\citep{3495724.3496298} and $0.125\pm0.011$ for conditional flow matching~\citep{lipman2023flow}. Other learned actions that share the same reference, from velocity regression and binary NRE, reach similar error ($0.0900$ and $0.0897$), so with dense observations the gain over the generators comes from optimizing a learned path objective. With endpoints only, the observations constrain the path less and the learned action matters more. Among action-based methods, LEAP then gives the lowest error, $0.6225$ compared to $0.6609$ for velocity regression and $0.6682$ for binary NRE. That said, the conditional generators, which are trained with completion masks, achieve lower error in this setting. LEAP is trained only on complete trajectories and uses the same learned action for both observation patterns, so the observation pattern can change without training a new model. \Cref{fig:alanine}a shows one example, with full results in \Cref{app:alanine_details}.

\textbf{Transfer across temperatures.} We keep the learned correction fixed and use the diffusion, temperature and reference process of the target system. Training at $600$\,K gives mean within-context AUROC $0.9579\pm0.0057$ at $500$\,K, compared to $0.9403\pm0.0099$ for direct $500$\,K training, likely because the $600$\,K training set is larger ($463$ vs $130$ trajectories). Training at $500$\,K gives $0.9719\pm0.0068$ at $600$\,K, compared to $0.9754\pm0.0012$ for direct $600$\,K training. The learned correction therefore remains useful after changing temperature. \looseness=-1

\textbf{Adding preferences and constraints.} Next, using the same learned $600$\,K action, we change the inference objective. Increasing the basin preference weight from $0$ to $20$ increases the mean preference $M$ from $0.0658$ to $0.2517$, while the projected transition cost increases from $0.9389$ to $0.9809$ nats per active step. We also impose $|\phi|\leq150^\circ$ while keeping the same basin preference. In \Cref{fig:alanine}c, the preference-only path enters the excluded angular region, while the constrained path stays within the bound and reaches the same endpoint. Thus, we can add preferences and geometric constraints without retraining the learned action. \looseness=-1

\begin{figure*}[t]
\centering
\captionsetup{font=small}
\begin{subfigure}[b]{0.3449\textwidth}
    \includegraphics[width=\linewidth]{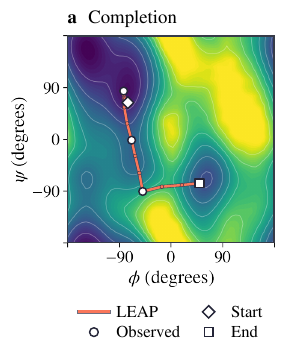}
    \phantomsubcaption\label{fig:alanine-completion}
\end{subfigure}%
\begin{subfigure}[b]{0.2758\textwidth}
    \includegraphics[width=\linewidth]{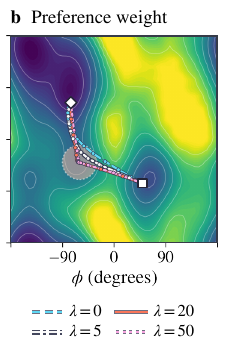}
    \phantomsubcaption\label{fig:alanine-preference}
\end{subfigure}%
\begin{subfigure}[b]{0.3793\textwidth}
    \includegraphics[width=\linewidth]{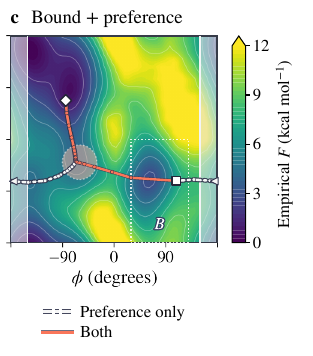}
    \phantomsubcaption\label{fig:alanine-compose}
\end{subfigure}
\vspace{-32pt}
\caption{\textbf{Reuse of a fixed molecular action.}
All panels use the same $600$\,K LEAP action.
(a) Completion from endpoints and three interior observations.
(b) Increasing the basin preference raises $M$ from $0.23$ to $0.81$ for the shown path.
(c) Adding $|\phi|\leq150^\circ$ changes the route while preserving the terminal condition.
Diamonds and squares are endpoints, dots are saved frames and the background is empirical free energy. \looseness=-1}
\label{fig:alanine}
\vspace{-15pt}
\end{figure*}

\section{Conclusion} \label{sec:discussion}
\vspace{-10pt}

In this work, we introduced \emph{action learning}, which treats the stochastic action as a reusable learning target for trajectory data, and LEAP, which learns this action by comparing complete trajectories. In the double well setting, the learned action recovers the underlying physical drift and supports path inference under new endpoints and horizons. In the orbital benchmark, it captures interactions across trajectories induced by a persistent latent regime that cannot be represented by additive transition models. In molecular coordinates, the same learned action enables completion, temperature transfer, basin preferences and angular constraints without retraining. These experiments show that learning an objective on path space provides a common interface between trajectory data and diverse downstream inference problems, including settings where the reduced dynamics are unavailable in closed form. Looking forward, we hope action learning encourages a broader view for trajectory modeling, in particular, shifting the focus to learning reusable objectives over complete paths.
\looseness=-1

\textbf{Limitations.}
LEAP requires sufficient overlap between the reference and observed path distributions, as poor overlap makes the learned correction harder to estimate accurately. While the structured correction in \Cref{eq:ratio_param} is exact for Gaussian Markov processes with shared diffusion, state-dependent diffusion requires more likelihood ratio terms. Finally, projecting dynamics onto reduced coordinates can introduce memory and friction or diffusion that depends on the state~\citep{Ma_2016,Vroylandt_2022}. Our alanine experiments therefore learn an effective action in $(\phi,\psi)$ rather than the full molecular action
\looseness=-1

\section*{Impact statement}
This paper presents work whose goal is to advance the field of Machine Learning. Learned trajectory objectives may help analyze and steer dynamical systems in science, robotics, planning and other sequential domains. As with other learned objectives, they could also be used to steer systems toward harmful outcomes, so deployment requires domain validation and careful attention to the data.\looseness=-1

\section*{Acknowledgements}
SL is supported by a Kempner Graduate Fellowship.

\bibliography{references}
\bibliographystyle{iclr2027_conference}

\appendix

\clearpage

\section*{Appendix}
\label{app:roadmap}
\numberwithin{equation}{section}

The appendix is organized as follows. \Cref{app:theory} presents the theoretical results on action recovery, Gaussian drift corrections, conditioning, nonadditive actions and downstream inference. \Cref{app:protocols} describes the shared evaluation metrics, training procedures and inference protocols. \Cref{app:wells} provides additional details on the double well experiments, while \Cref{app:reference_studies} examines reference choice, temporal resolution and training budget. \Cref{app:orbital} presents the orbital results, model comparisons and scaling experiments. Finally, \Cref{app:reuse_studies} details the molecular experiments, including completion, temperature transfer, preferences, constraints and additional validation.

\section{Theory and derivations} \label{app:theory}

We additionally provide the assumptions and derivations in~\Cref{sec:theory}, which includes population action recovery, the structured Gaussian correction and drift interpretation, matched trajectory selection, nonadditive actions and downstream stability.

\textbf{Notation.} Fix a context $c$ and let $\Omega_c$ be the trajectory space. The densities $p_{\mathrm{data}}(\cdot\mid c)$ and $p_{\mathrm{ref}}(\cdot\mid c)$ are defined with respect to the same base measure on $\Omega_c$. We use the target correction $G^*$ and target action $S^*$ defined in \Cref{eq:path_ratio_target,eq:target_action}.

\subsection{Action recovery from trajectory comparisons}
\label{app:bayes_ratio}

Recall that LEAP learns the correction by identifying one observed trajectory among $K$ reference trajectories drawn under the same context. At the population level, this objective recovers the trajectory density ratio for every $K\geq1$.

Fix a context $c$ and write $p=p_{\mathrm{data}}(\cdot\mid c)$ and $q=p_{\mathrm{ref}}(\cdot\mid c)$, so that \[ G^*(\Phi;c)=\log\frac{p(\Phi)}{q(\Phi)}. \] Draw an index $I$ uniformly from $\{0,\ldots,K\}$. Conditional on $I=i$, independently draw \begin{equation} \label{eq:app_set_sampling} \Phi_i\sim p, \qquad \Phi_j\sim q,\quad j\neq i. \end{equation} The corresponding population objective is \begin{equation} \label{eq:app_setwise_objective} \mathcal L_K(G;c) = -\mathbb E\left[ \log \frac{\exp G(\Phi_I;c)} {\sum_{j=0}^{K}\exp G(\Phi_j;c)} \,\middle|\,c \right]. \end{equation} By symmetry, this is the population form of \Cref{eq:cl_main}.

\begin{theorem}[Population action recovery] \label{thm:app_action_recovery} Suppose $p$ and $q$ are positive on the same support, the candidates in \Cref{eq:app_set_sampling} are conditionally independent and the model class contains $G^*$. Then, for every $K\geq1$, every population minimizer satisfies \begin{equation} \label{eq:app_score_recovery} G_{\theta^*}(\Phi;c) = G^*(\Phi;c)+a(c) \end{equation} almost everywhere, where $a(c)$ is independent of $\Phi$. If $\mathbb E_q|G^*|<\infty$, then the reference-centered action satisfies \begin{equation} \label{eq:app_action_recovery} S_{\theta^*}(\Phi;c) = S^*(\Phi;c)+\beta(c) \end{equation} for a constant $\beta(c)$ independent of $\Phi$. \end{theorem}

\begin{proof} By Bayes' rule, \begin{equation} \label{eq:app_set_posterior} \begin{aligned} \mathbb P(I=i\mid\Phi_{0:K},c) &= \frac{ p(\Phi_i)\prod_{j\neq i}q(\Phi_j) }{ \sum_{\ell=0}^{K} p(\Phi_\ell)\prod_{j\neq\ell}q(\Phi_j) }= \frac{p(\Phi_i)/q(\Phi_i)} {\sum_{\ell=0}^{K}p(\Phi_\ell)/q(\Phi_\ell)} = \frac{\exp G^*(\Phi_i;c)} {\sum_{\ell=0}^{K}\exp G^*(\Phi_\ell;c)}. \end{aligned} \end{equation} The comparison loss is minimized when the model matches this posterior. Therefore, for any two candidates $\Phi_i$ and $\Phi_j$, $G_{\theta^*}(\Phi_i;c)-G_{\theta^*}(\Phi_j;c) = G^*(\Phi_i;c)-G^*(\Phi_j;c), $ which implies $G_{\theta^*}(\Phi;c) = G^*(\Phi;c)+a(c)$ almost everywhere. Then, the constant $a(c)$ cancels from the softmax.

Reference centering then gives \begin{equation} \label{eq:app_centered_correction} \begin{aligned} \widetilde G_{\theta^*}(\Phi;c) &= G_{\theta^*}(\Phi;c) - \mathbb E_q G_{\theta^*}(\Phi';c) = G^*(\Phi;c) - \mathbb E_q G^*(\Phi';c). \end{aligned} \end{equation} Substituting into \Cref{eq:anchored_energy}, \begin{equation} \label{eq:app_data_action} \begin{aligned} S_{\theta^*}(\Phi;c) &= -\log q(\Phi) - \widetilde G_{\theta^*}(\Phi;c) = -\log p(\Phi) + \mathbb E_q G^*(\Phi';c) = S^*(\Phi;c)+\beta(c), \end{aligned} \end{equation} where $\beta(c) = \mathbb E_q G^*(\Phi';c)-\alpha(c).$\end{proof}

Reference centering selects an additive convention for the learned correction. With finite data, correction accuracy then depends on reference coverage, the number of reference trajectories, model capacity and optimization.

\subsection{Structured correction and drift identification} \label{app:gaussian} \label{app:em_ratio}

Consider the Euler--Maruyama process \begin{equation} \label{eq:app_em_dynamics} x_{k+1} = x_k+b(x_k,t_k;c)\Delta t +\sqrt{2D\Delta t}\,\xi_k, \qquad \xi_k\overset{\mathrm{iid}}{\sim}\mathcal N(0,I). \end{equation} Let $v_k=(x_{k+1}-x_k)/\Delta t$. Its transition density is \begin{equation} \label{eq:app_em_transition_density} q_b(x_{k+1}\mid x_k,c) = (4\pi D\Delta t)^{-d/2} \exp\left[ -\frac{\Delta t}{4D} \|v_k-b(x_k,t_k;c)\|^2 \right], \end{equation} so for a fixed initial state, \begin{equation} \label{eq:em_action} -\log p_b(\Phi\mid c) = C_{\mathrm{EM}}(c) + \sum_{k=0}^{N-1} \frac{\Delta t}{4D} \|v_k-b(x_k,t_k;c)\|^2, \end{equation} where $C_{\mathrm{EM}}(c)=\frac{Nd}{2}\log(4\pi D\Delta t)$.

Suppose the data and reference processes share the diffusion, time grid and initial law and differ via $b_{\mathrm{data}}=b_{\mathrm{ref}}+u^*$. Writing $u_k=u^*(x_k,t_k;c)$, the Gaussian likelihood ratio gives \begin{equation} \label{eq:app_gaussian_ratio} G^*(\Phi;c) = \sum_{k=0}^{N-1} \left[ \frac{ u_k^\top \bigl(v_k-b_{\mathrm{ref}}(x_k,t_k;c)\bigr) }{2D} - \frac{\|u_k\|^2}{4D} \right]\Delta t . \end{equation} Recall that this is the structured correction in \Cref{eq:ratio_param}. Subtracting it from the reference action gives \begin{equation} \label{eq:app_quadratic_action} S_{\mathrm{ref}}(\Phi;c)-G^*(\Phi;c) = C_{\mathrm{EM}}(c) + \sum_{k=0}^{N-1} \frac{\Delta t}{4D} \left\| v_k-b_{\mathrm{data}}(x_k,t_k;c) \right\|^2. \end{equation} Thus the structured correction parameterizes the drift difference exactly under the shared-diffusion assumptions. This discrete likelihood ratio is the Euler--Maruyama analogue of the shared-diffusion Girsanov likelihood ratio~\citep{doi:10.1137/1105027}.

\begin{proposition}[Drift identification under shared diffusion] \label[proposition]{prop:drift_identification} Let $P_b$ and $P_{\widetilde b}$ be the unconditioned Euler--Maruyama trajectory distributions with the same initial law, step size $\Delta t>0$ and diffusion coefficient $D>0$. Then \begin{equation} \label{eq:app_path_kl} \operatorname{KL}(P_b\|P_{\widetilde b}) = \frac{\Delta t}{4D} \sum_{k=0}^{N-1} \mathbb E_{P_b} \left[ \left\| b(X_k,t_k;c)-\widetilde b(X_k,t_k;c) \right\|^2 \right]. \end{equation} Therefore, equality of the trajectory distributions implies equality of their drifts at states and times visited by the process. \end{proposition}

\begin{proof} The KL chain rule decomposes the path divergence into expected transition divergences. Because the corresponding Gaussian transitions have covariance $2D\Delta t\,I$, \[ \operatorname{KL} \left( q_b(\cdot\mid X_k,c) \| q_{\widetilde b}(\cdot\mid X_k,c) \right) = \frac{\Delta t}{4D} \left\| b(X_k,t_k;c)-\widetilde b(X_k,t_k;c) \right\|^2. \] Summing over $k$ gives \Cref{eq:app_path_kl}. If the path distributions are equal, the KL divergence is zero, so the drift difference vanishes almost surely on every occupied state-time pair. \end{proof}

For structured LEAP, $b=b_{\mathrm{ref}}+u^*$ and $\widetilde b=b_{\mathrm{ref}}+u_{\theta^*}$. Recovering the unconditioned path distribution then identifies $u_{\theta^*}=u^*$ on the support of the observed states and times.

\subsection{Matched conditioning and selected trajectories}
\label{app:matched_conditioning}

Trajectory datasets are often restricted to a fixed endpoint or terminal event. When the data and reference use the same condition, conditioning changes their path log-likelihood ratio only by a constant within that condition.

\begin{proposition}[Matched conditioning preserves action differences]
\label[proposition]{prop:matched_conditioning}
Let $Y=\tau(\Phi)$ be a trajectory statistic. Suppose the data and reference admit conditional densities given $Y=y$ with respect to the same base measure and both assign positive density to $Y=y$. Then, for almost every $\Phi$ with $\tau(\Phi)=y$,
\begin{equation}
\label{eq:app_matched_conditioning}
\log \frac{p_{\mathrm{data}}(\Phi\mid Y=y,c)}
     {p_{\mathrm{ref}}(\Phi\mid Y=y,c)} =
\log \frac{p_{\mathrm{data}}(\Phi\mid c)}
     {p_{\mathrm{ref}}(\Phi\mid c)}
+
\log
\frac{p_{\mathrm{ref}}^Y(y\mid c)}
     {p_{\mathrm{data}}^Y(y\mid c)}.
\end{equation}
\end{proposition}

\begin{proof}
For $\tau(\Phi)=y$,
\[
p_{\mathrm{data}}(\Phi\mid Y=y,c)
=
\frac{p_{\mathrm{data}}(\Phi\mid c)}
     {p_{\mathrm{data}}^Y(y\mid c)},
\qquad
p_{\mathrm{ref}}(\Phi\mid Y=y,c)
=
\frac{p_{\mathrm{ref}}(\Phi\mid c)}
     {p_{\mathrm{ref}}^Y(y\mid c)}.
\]
Taking their log ratio gives \Cref{eq:app_matched_conditioning}.
The last term depends on $y$ and $c$. Therefore conditioning preserves action differences between trajectories with the same condition.
\end{proof}

Setting $Y=X_N$ gives exact endpoint conditioning and setting $Y=\mathbbm{1}\{X_N\in B\}$ and conditioning on $Y=1$ gives conditioning on a terminal region.

Conditioning preserves complete path likelihood ratios up to a constant, but it changes the one-step transition distribution. For the terminal event $\mathcal E=\{X_N\in B\}$, define $h_k(x;c)=P(\mathcal E\mid X_k=x,c)$. The conditioned transition kernel is
\begin{equation}
\label{eq:app_doob_kernel}
q_k^{\mathcal E}(x'\mid x,c)
=
q_k(x'\mid x,c)
\frac{h_{k+1}(x';c)}{h_k(x;c)}.
\end{equation}
For smooth positive $h$ in continuous time, the corresponding drift is
\begin{equation}
\label{eq:app_conditioned_drift}
b^{\mathcal E}(x,t)
=
b(x,t)+2D\nabla\log h(x,t).
\end{equation}
Thus local regression on selected trajectories estimates $b^{\mathcal E}$.

\begin{remark}[Drift identification under fixed endpoints]
Matched conditioning alone does not identify the unconditioned drift. Consider constant drift $b(x,t)=v$ relative to zero drift. The path log-likelihood ratio is
\begin{equation}
\label{eq:app_constant_drift_bridge}
\begin{aligned}
\log\frac{p_v(\Phi\mid x_0)}{p_0(\Phi\mid x_0)}
&=
\sum_{k=0}^{N-1}
\left[
\frac{v^\top(x_{k+1}-x_k)}{2D}
-
\frac{\Delta t\|v\|^2}{4D}
\right] =
\frac{v^\top(x_N-x_0)}{2D}
-
\frac{T\|v\|^2}{4D}.
\end{aligned}
\end{equation}
The intermediate states cancel from the ratio, which leaves only the initial and final states $(x_0,x_N)$. When both endpoints are fixed, the ratio is therefore the same for every trajectory connecting them, so it is not possible to distinguish different values of $v$ from these trajectories alone. Thus, \Cref{prop:drift_identification} establishes drift identification for unconditioned trajectory distributions. We separately evaluate drift recovery from trajectories selected by their final states in \Cref{sec:doublewell}.
\end{remark}

\subsection{Nonadditive action from a persistent latent regime} \label{app:memory_theory}

A latent physical parameter may remain fixed for an entire trajectory without appearing in the observed state. Marginalizing this parameter can introduce interactions between transitions, resulting in an action that cannot simply be expressed as a sum of individual transition costs.

Let $z\in\{-1,+1\}$ be sampled once with equal probability and held fixed for the complete trajectory. Conditional on $z$, assume the path density factorizes as $p_z(\Phi\mid c) = \prod_{k=0}^{N-1} q_{z,k}(x_{k+1}\mid x_k,c),$ with reference $p_0(\Phi\mid c)=\prod_{k=0}^{N-1}q_{0,k}(x_{k+1}\mid x_k,c)$. Any shared initial density cancels from the likelihood ratio.

Define $\ell_{z,k}=\log q_{z,k}(x_{k+1}\mid x_k,c)$ and $\ell_{0,k}=\log q_{0,k}(x_{k+1}\mid x_k,c)$ and let \begin{equation} \label{eq:app_latent_regime_coordinates} d_k = \frac{\ell_{+1,k}-\ell_{-1,k}}{2}, \qquad m_k = \frac{\ell_{+1,k}+\ell_{-1,k}}{2} - \ell_{0,k}. \end{equation}

\begin{proposition}[Marginalizing latent dynamics induces nonadditive actions] \label[proposition]{prop:trajectory_memory} Let $p_{\mathrm{data}}(\Phi\mid c) = \frac12 p_{+1}(\Phi\mid c) + \frac12 p_{-1}(\Phi\mid c)$. Then \begin{equation} \label{eq:memory_correction_proof} G^*(\Phi;c) = \sum_{k=0}^{N-1} m_k + \log\cosh\left( \sum_{k=0}^{N-1} d_k \right). \end{equation} For any $i\neq j$, \begin{equation} \label{eq:memory_interaction} \frac{\partial^2 G^*} {\partial d_i\,\partial d_j} = \operatorname{sech}^2\left( \sum_{k=0}^{N-1} d_k \right) >0. \end{equation} Therefore, on any region where $d_i$ and $d_j$ can vary independently, $G^*$ cannot be represented by a twice differentiable additive model $G(\Phi;c)=\sum_k g_k(d_k,m_k)$. \end{proposition}

\begin{proof} Using the conditional factorization, \[ G^*(\Phi;c) = \log\left[ \frac12 \exp\left(\sum_k \ell_{+1,k}-\ell_{0,k}\right) + \frac12 \exp\left(\sum_k \ell_{-1,k}-\ell_{0,k}\right) \right]. \] By \Cref{eq:app_latent_regime_coordinates}, $\ell_{+1,k}-\ell_{0,k}=m_k+d_k$ and $\ell_{-1,k}-\ell_{0,k}=m_k-d_k$. Therefore, \[ \begin{aligned} G^*(\Phi;c) &= \sum_k m_k + \log\left[ \frac12 e^{\sum_k d_k} + \frac12 e^{-\sum_k d_k} \right] = \sum_k m_k + \log\cosh\left(\sum_k d_k\right). \end{aligned} \] Let $A=\sum_k d_k$. Then $\partial G^*/\partial d_i=\tanh A$, so for $i\neq j$, $\frac{\partial^2G^*}{\partial d_i\,\partial d_j} = \operatorname{sech}^2 A >0.$

In contrast, any additive model satisfies $\partial^2(\sum_k g_k(d_k,m_k))/(\partial d_i\,\partial d_j)=0$ for $i\neq j$.\end{proof}


\subsection{Downstream stability under action reuse} \label{app:stability}

Fix a context $c$ and feasible set $\Gamma(c)$ and suppose \begin{equation} \label{eq:uniform_error} \sup_{\Phi\in\Gamma(c)} \left| S_\theta(\Phi;c)-S^*(\Phi;c)-\beta(c) \right| \leq \varepsilon . \end{equation} For any added path functional $H$, define $F_\theta=S_\theta+H$ and $F^*=S^*+H$.

\begin{theorem}[Stability under action reuse] \label{thm:reuse_stability} If $\widehat\Phi$ is $\delta$-suboptimal for $F_\theta$, then \begin{equation} \label{eq:app_map_stability} F^*(\widehat\Phi;c) - \inf_{\Phi\in\Gamma(c)}F^*(\Phi;c) \leq 2\varepsilon+\delta . \end{equation} If $\pi_\theta\propto e^{-F_\theta}$ and $\pi^*\propto e^{-F^*}$ have finite positive normalizing constants, then \begin{equation} \label{eq:app_density_stability} \left| \log \frac{\pi_\theta(\Phi\mid c)} {\pi^*(\Phi\mid c)} \right| \leq 2\varepsilon \qquad \forall \Phi\in\Gamma(c). \end{equation} \end{theorem}

\begin{proof} Write $e(\Phi)=S_\theta(\Phi;c)-S^*(\Phi;c)-\beta(c)$, so $|e(\Phi)|\leq\varepsilon$. If $\widehat\Phi$ is $\delta$-suboptimal for $F_\theta$, then for every $\Phi\in\Gamma(c)$, $F^*(\widehat\Phi;c) \leq F^*(\Phi;c) + e(\Phi)-e(\widehat\Phi) +\delta \leq F^*(\Phi;c)+2\varepsilon+\delta,$ which gives \Cref{eq:app_map_stability} after taking the inf.

For the normalized distributions, \[ \log\frac{\pi_\theta(\Phi\mid c)}{\pi^*(\Phi\mid c)} = -e(\Phi) - \log\mathbb E_{\pi^*}[e^{-e(\Phi)}]. \] Since $|e|\leq\varepsilon$, the second term has magnitude at most $\varepsilon$, giving \Cref{eq:app_density_stability}. \end{proof}

Because the same $H$ appears in both objectives, these bounds continue to hold after adding new preferences or constraints.

\section{Experimental details} \label{app:protocols}

We describe the evaluation metrics, training settings and inference procedures used across experiments in \Cref{sec:experiments}. 

\textbf{Compute resources.}
We run each training job on a single NVIDIA GPU (H200, H100, A100) with up to 8 CPU cores and 64 GB of memory. The reference study in~\Cref{app:reference_studies} runs on CPUs.

\subsection{Metrics and uncertainty} \label{app:metrics}

\textbf{Drift recovery.} When the true drift is known, we report \begin{equation} \label{eq:app_drift_mse} \operatorname{MSE}_{b} = \frac{1}{n_x} \sum_{i=1}^{n_x} \left\| \widehat b(x_i)-b^*(x_i) \right\|^2. \end{equation} We evaluate all methods on the same $n_x$ states. For the double well, we use $301$ states uniformly spanning $[-1,1]$.

\textbf{Action correction error.} Because the action is identified up to a context-dependent constant, we compare correction differences within each context. For $i,j\in I_c$, let $\Delta G_{ij}=G(\Phi_i;c)-G(\Phi_j;c)$. We define \begin{equation} \label{eq:correction_metric} \mathrm{NRMSE}_{\Delta G} = \left[ \frac{ \sum_c n_c^{-2} \sum_{i,j\in I_c} \left( \Delta G_{\theta,ij}-\Delta G^*_{ij} \right)^2 }{ \sum_c n_c^{-2} \sum_{i,j\in I_c} \left( \Delta G^*_{ij} \right)^2 } \right]^{1/2}. \end{equation} Exact correction differences give $0$ and a constant prediction gives $1$.

\textbf{Trajectory discrimination and ordering.} We report within-context AUROC for separating data and reference trajectories. The metric does not measure correction scale, so we also use $\mathrm{NRMSE}_{\Delta G}$ when we need numerical action differences.

\textbf{Path quality.} When the true drift is known, we evaluate inferred paths with \begin{equation} \label{eq:well_true_cost} Q_{\mathrm{true}}(\Phi) = \frac{\Delta t}{4} \sum_{k=0}^{N-1} \left\| \frac{x_{k+1}-x_k}{\Delta t} - b^*(x_k) \right\|^2. \end{equation} For the double well, $b^*(x)=x-x^3$. We use this cost only for evaluation; each method performs inference with its own objective.

\textbf{Uncertainty.} We use at least three independent random seeds for every experiment. For all the tables, we report the sample mean$\pm$sample standard deviation.

\subsection{Path inference} \label{app:inference}

\begin{algorithm}[ht]
\caption{Path optimization with a learned LEAP action}
\label{alg:leap_inference}
\begin{algorithmic}[1]
\REQUIRE learned action $S_\theta(\Phi;c)$,
initialization $\Phi^{(0)}\in\Gamma(c)$,
optional rewards $\{M_j\}$,
optional penalties $\{C_\ell\}$
\STATE define
\[
F(\Phi;c)
=
S_\theta(\Phi;c)
-
\sum_j \gamma_j M_j(\Phi)
+
\sum_\ell \lambda_\ell C_\ell(\Phi)
\]
\STATE initialize $\Phi\leftarrow\Phi^{(0)}$
\FOR{$r=1$ to $R$}
    \STATE update the free trajectory coordinates by one optimizer step on
    $F(\Phi;c)$
    \STATE project onto hard constraints when required
    \IF{using knot redistribution}
        \STATE redistribute knots and reevaluate $F$ before the next optimizer step
    \ENDIF
\ENDFOR
\RETURN optimized path $\Phi^*$
\end{algorithmic}
\end{algorithm}

\subsection{Training and model selection} \label{app:training}

\Cref{tab:training_design} summarizes the main settings. Additionally, for the double well, we use a potential network with three hidden layers of width $512$, GELU activations and LayerNorm. We train with AdamW, weight decay $10^{-5}$, cosine learning rate decay to $10^{-5}$ and gradient clipping at $1$. For the orbital experiments, we use validation stopping with patience $20$. \Cref{app:alanine_details} gives the molecular architecture, reference and regularization settings.

\begin{table}[t]
\centering
\small
\begin{tabular}{@{}lccc@{}}
\toprule
& Double well & Orbital & Alanine \\
\midrule
Training paths & $5{,}000$ & $1{,}016$ & $130$ / $463$ \\
Reference paths $K$ & $8$ & $4$ & $8$ \\
Training budget & $500$ epochs & $100$ epochs & $1{,}200$ updates \\
Batch size & $64$ & $8$ & $32$ \\
Learning rate & $10^{-3}$ & $3\times10^{-4}$ & $3\times10^{-4}$ \\
Training seeds & $3$ & $3$ & $3$ \\ \bottomrule 
\end{tabular} \caption{\textbf{Training settings for the principal experiments.}
Molecular counts correspond to the $500$\,K and $600$\,K training sets.
Orbital model inputs are described in
\Cref{app:orbital_comparisons}.}\label{tab:training_design}\end{table}

\section{Double well experiments} \label{app:wells}

We provide additional details on the double well experiments, including trajectory generation, baselines, path inference and the analytic cost for long horizons.

\subsection{Trajectory construction and baselines} \label{app:well_protocol}

\textbf{Data and reference.} We simulate from $x_0=-1$ to $T=3$ with $D=0.20$ and $\Delta t=0.01$ and retain trajectories with $x_T\in[0.9,1.1]$. Each seed uses $5{,}000$ selected training trajectories. The reference is Brownian motion with the same diffusion and terminal condition. We sample $x_T$ from its terminal density restricted to $[0.9,1.1]$, then sample the corresponding Gaussian bridge. The data and reference use the same terminal selection event as in \Cref{prop:matched_conditioning}.

\textbf{LEAP model.} We parameterize the drift by a scalar potential, $b_\theta(x)=-dU_\theta(x)/dx$. Training settings are given in \Cref{tab:training_design}.

\textbf{Local drift baselines.} Kramers-Moyal averages $(x_{k+1}-x_k)/\Delta t$ in state bins of width $2/300$. Cubic velocity regression fits the same local target by least squares using $\{1,x,x^2,x^3\}$. This basis contains the exact drift $x-x^3$, so its error does not come from function class misspecification. Both methods use the same selected trajectories as LEAP. By \Cref{app:matched_conditioning}, terminal selection changes their local regression target.

\begin{table}[t]
\centering
\small
\begin{tabular}{@{}ll@{}}
\toprule
Method & Definition \\
\midrule
Kramers-Moyal & Binned mean local increment \\
Cubic velocity regression & Cubic least-squares fit to local increments \\
Zero field & $b(x)=0$ \\
Geometric path & Fixed waiting and linear crossing path \\
Analytic drift & Exact drift and action \\
LEAP & Complete path comparison objective \\
\bottomrule
\end{tabular}
\caption{\textbf{Double well methods and controls.}}
\label{tab:well_baselines}
\end{table}

\subsection{Path inference under new boundary conditions} \label{app:well_path_inference}

\textbf{Inference setting.} We keep the learned potential fixed and change the endpoint and horizon. Training paths end in $[0.9,1.1]$ at $T=3$, whereas inference fixes $x_0=-1$ and $x_T=0$. We evaluate $T\in\{1,4,8,16\}$ and optimize only the free path coordinates.

\textbf{Optimization.} We use the same four fixed linear and waiting initializations for every differentiable objective and run $500$ L-BFGS iterations. Each method selects the candidate with lowest value under its own objective. We then evaluate the returned path using the ground-truth cost $Q_{\mathrm{true}}$ from \Cref{eq:well_true_cost}. For LEAP, we differentiate through the learned drift during path optimization.

\subsection{Long-horizon barrier scale} \label{app:barrier_theory}

For \[ \mathcal S_D^T(\phi) = \frac{1}{4D} \int_0^T \left|\dot\phi(t)+U'(\phi(t))\right|^2dt, \] with $\phi(0)=-1$ and $\phi(T)=0$, the barrier $U(0)-U(-1)=1/4$ sets the long-horizon cost. Using $|\dot\phi+U'|^2=|\dot\phi-U'|^2+4U'\dot\phi$, \[ \begin{aligned} D\mathcal S_D^T(\phi) &= U(\phi(T))-U(\phi(0)) + \frac14\int_0^T |\dot\phi-U'(\phi)|^2dt  &\geq U(0)-U(-1) = \frac14. \end{aligned} \] The reverse gradient path $\dot\phi=U'(\phi)$ approaches this bound as $T\to\infty$, so \[ \lim_{T\to\infty} \inf_{\substack{\phi(0)=-1\\\phi(T)=0}} D\mathcal S_D^T(\phi) = \frac14. \]

This limit explains why the optimal action approaches the barrier value for long horizons. At finite horizon $T$, we evaluate the optimized discrete path under the true action.

\begin{table*}[t]
\centering
\captionsetup{font=small}
\scriptsize
\setlength{\tabcolsep}{5pt}

\begin{tabular*}{0.82\textwidth}{
    @{\extracolsep{\fill}}lccc@{}
}
\toprule
& Drift MSE $\downarrow$
& \multicolumn{2}{c}{True path cost $\downarrow$} \\
\cmidrule(l){3-4}
Method & & $T=4$ & $T=16$ \\
\midrule

Kramers-Moyal
& $0.8132{\scriptscriptstyle \pm 0.0030}$
& --- & --- \\

Cubic velocity reg.
& $0.7752{\scriptscriptstyle \pm 0.0043}$
& $0.3189{\scriptscriptstyle \pm 0.0008}$
& $0.6126{\scriptscriptstyle \pm 0.0079}$ \\

Zero field
& $0.0759$
& $0.2637$
& $0.4454$ \\

Geometric path
& ---
& $0.2637$
& $0.2637$ \\
Analytic drift
& $0$
& $0.2532$
& $0.2500$ \\

\midrule
\textbf{LEAP}
& $\mathbf{(3.3{\scriptscriptstyle \pm 2.7})\times10^{-4}}$
& $\mathbf{0.2533}{\scriptscriptstyle \mathbf{\pm 0.0001}}$
& $\mathbf{0.2506}{\scriptscriptstyle \mathbf{\pm 0.0006}}$ \\

\bottomrule
\end{tabular*}
\caption{\textbf{Physical recovery and path inference.} Drift error and path cost under new endpoints and horizons.}
\label{tab:dw-recovery2}
\vspace{-10pt}
\end{table*}

\section{Reference choice and robustness}
\label{app:reference_studies}
We test sensitivity to the reference process on an unconditioned double well with an exact discrete path density. We use $b(x)=x-x^3$, $D=0.12$ and $X_0\sim\mathcal N(-0.8,0.15^2)$. Each of the seeds uses $256$ training, $96$ validation and $128$ test trajectories. Unless varied below, we use $24$ transitions with $\Delta t=0.04$, $T=0.96$ and a zero-drift reference with $D_{\mathrm{ref}}=0.20$.

Transition NCE compares each observed next state with reference next states sampled from the same current state. Gaussian MLE maximizes the Gaussian one-step transition likelihood and learns a scalar diffusion coefficient. Both use the same drift network as LEAP. Because this setting is unconditioned and Markov, the transition factorization is exact and transition-level estimators are well specified, so we expect the three estimators to perform similarly. The point of this ablation is to study the sensitivity to the reference; the advantages of LEAP appear under terminal selection (\Cref{sec:doublewell}) and persistent latent variables (\Cref{sec:orbital}).

\begin{table}[t]
\centering
\small
\begin{tabular}{@{}lccc@{}}
\toprule
Setting & LEAP & Transition NCE & Gaussian MLE \\
\midrule
\multicolumn{4}{@{}l}{\textit{Reference diffusion}}\\
$D_{\mathrm{ref}}=0.06$
& $0.31{\scriptscriptstyle \pm 0.13}$
& $0.24{\scriptscriptstyle \pm 0.13}$
& $0.347{\scriptscriptstyle \pm 0.014}$ \\
$D_{\mathrm{ref}}=0.12$
& $0.58{\scriptscriptstyle \pm 0.16}$
& $0.64{\scriptscriptstyle \pm 0.12}$
& $0.644{\scriptscriptstyle \pm 0.084}$ \\
$D_{\mathrm{ref}}=0.20$
& $0.66{\scriptscriptstyle \pm 0.14}$
& $0.71{\scriptscriptstyle \pm 0.11}$
& $0.729{\scriptscriptstyle \pm 0.084}$ \\
$D_{\mathrm{ref}}=0.32$
& $0.795{\scriptscriptstyle \pm 0.089}$
& $0.77{\scriptscriptstyle \pm 0.10}$
& $0.855{\scriptscriptstyle \pm 0.062}$ \\
$D_{\mathrm{ref}}=0.48$
& $0.939{\scriptscriptstyle \pm 0.040}$
& $0.837{\scriptscriptstyle \pm 0.080}$
& $0.931{\scriptscriptstyle \pm 0.034}$ \\
\midrule
\multicolumn{4}{@{}l}{\textit{Reference drift } $b_{\mathrm{ref}}=a_{\mathrm{ref}}b$}\\
$a_{\mathrm{ref}}=0.25$
& $0.42{\scriptscriptstyle \pm 0.15}$
& $0.45{\scriptscriptstyle \pm 0.19}$
& $0.486{\scriptscriptstyle \pm 0.066}$ \\
$a_{\mathrm{ref}}=0.50$
& $0.223{\scriptscriptstyle \pm 0.063}$
& $0.26{\scriptscriptstyle \pm 0.11}$
& $0.356{\scriptscriptstyle \pm 0.028}$ \\
$a_{\mathrm{ref}}=0.75$
& $0.1255{\scriptscriptstyle \pm 0.0054}$
& $0.121{\scriptscriptstyle \pm 0.029}$
& $0.325{\scriptscriptstyle \pm 0.013}$ \\
$a_{\mathrm{ref}}=1.00$
& $0.079{\scriptscriptstyle \pm 0.044}$
& $0.0384{\scriptscriptstyle \pm 0.0087}$
& $0.3170{\scriptscriptstyle \pm 0.0062}$ \\
\bottomrule
\end{tabular}
\caption{\textbf{Reference choice.}
Correction NRMSE $\downarrow$. The diffusion sweep uses
$b_{\mathrm{ref}}=0$; the drift sweep uses $D_{\mathrm{ref}}=0.20$. Contrastive
models use $K=4$.}
\label{tab:sensitivity_reference}
\end{table}

\textbf{Reference choice.} Adding physical drift to the reference decreases LEAP correction NRMSE from $0.66$ at $a_{\mathrm{ref}}=0$ to $0.223$ at $a_{\mathrm{ref}}=0.5$ and $0.079$ at $a_{\mathrm{ref}}=1$ (\Cref{tab:sensitivity_reference}).

\begin{table}[t]
\centering
\small
\begin{tabular*}{\linewidth}{@{\extracolsep{\fill}}lccc@{}}
\toprule
Setting & LEAP & Transition NCE & Gaussian MLE \\
\midrule
$\Delta t=0.02$, $N=48$
& $0.828{\scriptscriptstyle \pm 0.096}$
& $0.769{\scriptscriptstyle \pm 0.081}$
& $0.730{\scriptscriptstyle \pm 0.033}$ \\
$\Delta t=0.04$, $N=24$
& $0.66{\scriptscriptstyle \pm 0.14}$
& $0.71{\scriptscriptstyle \pm 0.11}$
& $0.729{\scriptscriptstyle \pm 0.084}$ \\
$\Delta t=0.08$, $N=12$
& $0.67{\scriptscriptstyle \pm 0.15}$
& $0.671{\scriptscriptstyle \pm 0.070}$
& $0.66{\scriptscriptstyle \pm 0.13}$ \\
\bottomrule
\end{tabular*}
\caption{\textbf{Temporal resolution.}
Correction NRMSE $\downarrow$, mean $\pm$ sample SD at fixed $T=0.96$.}
\label{tab:sensitivity_temporal}
\end{table}

\textbf{Temporal and observation robustness.} Across the three temporal grids, the three estimators reach similar correction error (\Cref{tab:sensitivity_temporal}). Increasing the hidden integration resolution from $2$ to $8$ substeps changes LEAP AUROC from $0.9171$ to $0.9216$. Observation noise decreases AUROC from $0.9156$ on clean trajectories to $0.5948$ at $\sigma_{\mathrm{obs}}=0.10$. \looseness=-1

\textbf{Data and reference budget.} Increasing the training set from $32$ to $256$ trajectories decreases LEAP NRMSE from $0.87$ to $0.66$. Increasing $K$ from $1$ to $4$ improves the most, with small changes after $K=4$. Under the default setting, selecting the path with the lowest predicted action yields regret of $0.0075$ nats for LEAP, compared to $0.0230$ nats for Gaussian MLE, measured using the true action.

\section{Orbital experiments} \label{app:orbital}

We give the exact orbital correction and controls for path structure, likelihood parameterization and comparison objective, using the setting in~\Cref{sec:orbital}. The benchmark has $31$ bodies, a $186$-dimensional state and $32$ transitions. We split initial contexts chronologically and evaluate on held-out contexts. During training, we resimulate data and reference trajectories from the $254$ training initial states at every epoch, using $4$ data trajectories per state (or $16$ and $32$ in Table 10). Therefore, the reported training trajectory counts are per epoch.

\subsection{Exact trajectory action} \label{app:orbital_theory}

For each trajectory, we draw $z\in\{-1,+1\}$ uniformly and set the solar mass to $m_\odot(1+z\eta)$. The same unobserved $z$ remains fixed for every transition; the reference uses the nominal mass $z=0$.

Let $\mu_z(x)$ be the Velocity-Verlet update under regime $z$ and let $\Sigma$ be the shared transition covariance. Conditional on the initial context, \begin{equation} \label{eq:exact_orbital_mixture} \begin{aligned} \ell_z(\Phi) &= \sum_{k=0}^{N-1} \log\mathcal N \left( x_{k+1};\mu_z(x_k),\Sigma \right), \quad  G^*(\Phi) &= \log\left[ \tfrac12 e^{\ell_{+1}(\Phi)} + \tfrac12 e^{\ell_{-1}(\Phi)} \right] - \ell_0(\Phi). \end{aligned} \end{equation} Marginalizing the shared variable combines likelihoods over entire trajectories, producing the nonadditive correction in \Cref{app:memory_theory}.

\subsection{Path structure and estimator controls} \label{app:orbital_comparisons}

We compare model architectures and estimation objectives to study the importance of capturing dependencies across time. LEAP-Local sums corrections from individual transitions, while Transition NCE uses the same transition information without the structured LEAP parameterization. LEAP-Global learns a correction from the entire trajectory, using accumulated features derived from the nominal gravitational dynamics. These features use the sensitivity $J_k$ of the nominal transition mean to the solar mass and the standardized innovations $e_k$. Defining $\gamma_k=e_k^\top J_k$ and $I=\sum_k|J_k|^2$, the features are $Z=\sum_k\gamma_k/\sqrt{I}$, $Z^2$, $\log(1+I)$ and the mean of $\gamma_k^2/|J_k|^2$. They do not contain the latent regime $z$ or the exact correction $G^*$, but assume that the perturbed parameter, the solar mass, is known.

LEAP-Hybrid combines local and global corrections. We also compare against Gaussian MLE, which predicts a Gaussian distribution for the next state given the current state and time, and History Gaussian MLE, which additionally conditions on previous states using a causal GRU. Finally, we compare binary and categorical neural ratio estimation objectives to assess the effect of the training objective independently of the model architecture. Results are reported in~\Cref{tab:nre}.

\begin{table}[t]
\centering
\small

\begin{tabular*}{\linewidth}{@{\extracolsep{\fill}}llc@{}}
\toprule
Model & Input & $\mathrm{NRMSE}_{\Delta G}\downarrow$ \\
\midrule
LEAP-Local
& One transition
& $1.0013{\scriptscriptstyle \pm 0.0015}$ \\
Transition NCE
& One transition
& $1.0008{\scriptscriptstyle \pm 0.0013}$ \\
LEAP-Global
& Complete path + accumulated features
& $\mathbf{0.228}{\scriptscriptstyle \mathbf{\pm 0.011}}$ \\
LEAP-Hybrid
& Local + complete path
& $0.232{\scriptscriptstyle \pm 0.014}$ \\
Gaussian MLE
& State + time
& $11.92{\scriptscriptstyle \pm 0.66}$ \\
History Gaussian MLE
& Causal history
& $10.162{\scriptscriptstyle \pm 0.093}$ \\
\midrule
\multicolumn{3}{@{}l}{\textit{Objective control with complete path model}}\\
Binary NRE
& Path encoder
& $1.00000{\scriptscriptstyle \pm 3 \times 10^{-5}}$ \\
Binary NRE
& Transformer + accumulated features
& $\mathbf{0.175}{\scriptscriptstyle \mathbf{\pm 0.011}}$ \\
Categorical NRE (LEAP)
& Transformer + accumulated features
& $0.228{\scriptscriptstyle \pm 0.015}$ \\
\bottomrule
\end{tabular*}

\caption{\textbf{Orbital controls.}
Correction NRMSE $\downarrow$ on the $31$-body benchmark; a constant
prediction gives $1$. Bold is the lowest mean within each control block.}
\label{tab:nre}
\end{table}

\textbf{Path structure.} Transition-level models remain at NRMSE $1$, while LEAP-Global reaches $0.228$ (\Cref{tab:nre}). The Gaussian likelihood models also perform poorly, as the exact correction in \Cref{eq:exact_orbital_mixture} involves a mixture over a latent regime shared across the trajectory. These results therefore show the importance of modeling dependencies across time.

\textbf{Objective control.} We also assess whether action recovery depends on the choice of contrastive objective. Using the same Transformer and accumulated features, binary and categorical ratio estimation reach NRMSE $0.175$ and $0.228$, respectively, showing that both objectives can recover the underlying action correction. On the other hand, binary ratio estimation with a generic path encoder stays near NRMSE $1$. This shows that accurate action estimation can be achieved with different ratio estimation objectives, while incorporating relevant physical structure substantially improves recovery.

\subsection{Comparison and data budgets}
\label{app:orbital_ablations}
\begin{table}[t]
\centering
\small

\begin{tabular}{@{}lcc@{}}
\toprule
Setting & Training & $\mathrm{NRMSE}_{\Delta G}\downarrow$ \\
\midrule
$K=4$  & validation stop & $0.2208{\scriptscriptstyle \pm 0.0056}$ \\
$K=4$  & $275$ epochs    & $0.2207{\scriptscriptstyle \pm 0.019}$ \\
$K=16$ & validation stop & $0.147{\scriptscriptstyle \pm 0.017}$ \\
$K=32$ & validation stop & $0.138{\scriptscriptstyle \pm 0.032}$ \\
\midrule
\multicolumn{3}{@{}l}{\textit{Training data}}\\
Training paths & LEAP-Global & Binary NRE \\
\midrule
$4{,}064$
& $0.149{\scriptscriptstyle \pm 0.069}$
& $0.129{\scriptscriptstyle \pm 0.017}$ \\
$8{,}128$
& $0.110{\scriptscriptstyle \pm 0.033}$
& $0.119{\scriptscriptstyle \pm 0.015}$ \\
\bottomrule
\end{tabular}

\caption{\textbf{Comparison and data budgets.}
Correction NRMSE $\downarrow$. Increasing $K$ changes the number of reference
paths per observed trajectory; the lower panel increases the number of training paths. Separate runs from \Cref{tab:nre}.}
\label{tab:orbital_budget}
\end{table}

Increasing $K$ from $4$ to $16$ decreases LEAP-Global NRMSE from $0.2208$ to $0.147$, while extending the $K=4$ model to $275$ epochs leaves it at $0.2207$ (\Cref{tab:orbital_budget}). Increasing the training set from $4{,}064$ to $8{,}128$ paths further decreases LEAP-Global NRMSE from $0.149$ to $0.110$.

\subsection{Longer paths and larger systems} \label{app:orbital_scaling}

We vary the number of bodies and transitions and evaluate each setting against its exact mixture correction. At $42$ bodies and $256$ transitions, each path contains $64{,}764$ coordinates. LEAP-Global reaches NRMSE $0.111$, while LEAP-Local remains at $1.0015$ (\Cref{tab:orbital_scaling}). The global model therefore maintains low correction error as the state dimension and trajectory length increase.

\begin{table}[t]
\centering
\small
\begin{tabular}{@{}rrcc@{}}
\toprule
Transitions & Path coords. & LEAP-Local & LEAP-Global \\
\midrule
\multicolumn{4}{@{}l}{\textit{31 bodies}}\\
32
& $6{,}138$
& $0.9994{\scriptscriptstyle \pm 1.2\times10^{-3}}$
& $0.2584{\scriptscriptstyle \pm 0.0097}$ \\
64
& $12{,}090$
& $0.9985{\scriptscriptstyle \pm 7\times10^{-4}}$
& $0.127{\scriptscriptstyle \pm 0.015}$ \\
128
& $23{,}994$
& $1.0052{\scriptscriptstyle \pm 0.0066}$
& $0.113{\scriptscriptstyle \pm 0.027}$ \\
256
& $47{,}802$
& $0.9965{\scriptscriptstyle \pm 7\times10^{-4}}$
& $0.117{\scriptscriptstyle \pm 0.055}$ \\
\midrule
\multicolumn{4}{@{}l}{\textit{42 bodies}}\\
32
& $8{,}316$
& $1.0001{\scriptscriptstyle \pm 7\times10^{-4}}$
& $0.251{\scriptscriptstyle \pm 0.016}$ \\
64
& $16{,}380$
& $1.0011{\scriptscriptstyle \pm 1.6\times10^{-3}}$
& $0.144{\scriptscriptstyle \pm 0.024}$ \\
128
& $32{,}508$
& $0.9997{\scriptscriptstyle \pm 1.4\times10^{-3}}$
& $0.163{\scriptscriptstyle \pm 0.053}$ \\
256
& $64{,}764$
& $1.0015{\scriptscriptstyle \pm 1.7\times10^{-3}}$
& $0.111{\scriptscriptstyle \pm 0.010}$ \\
\bottomrule
\end{tabular}
\caption{\textbf{Scaling with system size and path length.}
Correction NRMSE $\downarrow$. Each configuration uses its exact mixture
correction and $K=4$. Separate runs from \Cref{tab:nre}.}
\label{tab:orbital_scaling}
\end{table}

\subsection{Trajectory diffusion baseline} \label{app:diffusion}

We adapt GeoTDM~\citep{han2024geometric} to the orbital benchmark and train separate trajectory diffusion models on the data and reference paths. We use the difference between their denoising or variational objectives as an estimate of $G^*$. In \Cref{tab:diffusion}, both give chance AUROC ($0.4999$) and the variational estimate has Spearman correlation $0.0066$ with $G^*$. Thus trajectories with larger exact corrections do not systematically receive larger GeoTDM estimates.

\begin{table}[t]
\centering
\small
\begin{tabular}{@{}lccc@{}}
\toprule
Objective difference & Error $\downarrow$ & AUROC $\uparrow$ & Spearman $\uparrow$ \\
\midrule
Denoising
& $1.0000{\scriptscriptstyle \pm 2.3\times10^{-6}}$
& $0.4999{\scriptscriptstyle \pm 2\times10^{-4}}$
& --- \\
Variational
& $0.9883{\scriptscriptstyle \pm 1\times10^{-4}}$
& $0.4999{\scriptscriptstyle \pm 2\times10^{-4}}$
& $0.0066{\scriptscriptstyle \pm 0.0055}$ \\
\bottomrule
\end{tabular}
\caption{\textbf{Trajectory diffusion baseline.}
GeoTDM objective differences on the orbital benchmark with $31$ bodies. The denoising objective is evaluated using correction NRMSE, while the variational objective is evaluated using affine relative error.}
\label{tab:diffusion}
\end{table}

\section{Molecular inference and action reuse} \label{app:reuse_studies} \label{app:alanine_details}
Below we give additional details for the setting in~\Cref{sec:reuse_results}.

\subsection{Molecular data and models}

\textbf{Data and reference.} We use $130/17/22$ train/validation/test trajectories at $500$\,K and $463/73/77$ at $600$\,K, with simulation replicas assigned to disjoint splits. Vacuum Langevin dynamics use a $1$\,fs integration step and we save the backbone angles $(\phi,\psi)$ every $0.1$\,ps. Paths have at most $201$ frames.

We estimate a binned angular drift $\widehat b(x)$ and effective diffusion $D_{\mathrm{eff}}$ from the training replicas and define \[ b_{\mathrm{ref}}(x,t) = \widehat b(x) + 15\frac{\operatorname{wrap}(x_B-x)} {T-t+2\Delta t}. \] We condition reference paths on the same reactive event as the data and evaluate wrapped Gaussian transition densities by summing periodic images.

\textbf{Models and controls.} We embed angles periodically and parameterize LEAP with a drift network of three width $128$ hidden layers and a scalar residual with two width $64$ hidden layers. LEAP uses $K=8$ reference paths drawn from $32$ successful references per initial state. Binary NRE uses the same inputs and network capacity, while velocity regression fits wrapped one-step increments with the same drift network.

For conditional generation, we compare with DDPM~\citep{3495724.3496298} and flow matching~\citep{lipman2023flow} using the same molecular coordinates, observation masks and active durations. Both use a temporal convolutional network with width $96$ and seven residual layers. DDPM uses a $100$-step cosine schedule and flow matching uses $100$ Euler steps at inference.

\subsection{Completion controls} \label{app:alanine_completion}
We evaluate two conditioning settings. In the first, we observe the endpoints together with three interior frames; in the second, we observe only the endpoints and trajectory duration. For methods that optimize a learned action, we use eight shared initializations, run $400$ steps of projected Adam and select the path with the lowest action. Conditional generators produce eight candidate trajectories, which are selected using the reference action. In both settings, candidate selection does not use the held out target frames.
\begin{table}[t]
\centering
\small
\begin{tabular}{@{}lcc@{}}
\toprule
Method & Sparse completion & Endpoints only \\
\midrule
Reference action
& $0.1031{\scriptscriptstyle \pm 0.0016}$
& $0.8183{\scriptscriptstyle \pm 0.0042}$ \\
Untrained action
& $0.0947{\scriptscriptstyle \pm 0.0000}$
& $0.7953{\scriptscriptstyle \pm 0.0027}$ \\
Velocity regression
& $0.0900{\scriptscriptstyle \pm 0.0003}$
& $0.6609{\scriptscriptstyle \pm 0.0032}$ \\
Binary NRE
& $0.0897{\scriptscriptstyle \pm 0.0002}$
& $0.6682{\scriptscriptstyle \pm 0.0118}$ \\
Conditional DDPM
& $0.119{\scriptscriptstyle \pm 0.012}$
& $0.409{\scriptscriptstyle \pm 0.034}$ \\
Conditional flow matching
& $0.125{\scriptscriptstyle \pm 0.011}$
& $0.350{\scriptscriptstyle \pm 0.012}$ \\
LEAP
& $0.0931{\scriptscriptstyle \pm 0.0016}$
& $0.6225{\scriptscriptstyle \pm 0.0080}$ \\
\bottomrule
\end{tabular}
\caption{\textbf{Molecular completion at $600$\,K.}
Circular MSE $\downarrow$ on unobserved frames for the same observation
patterns and durations.}
\label{tab:alanine_completion}
\end{table}

With endpoints and three interior observations, LEAP achieves circular MSE $0.0931$, compared with $0.119$ for conditional DDPM and $0.125$ for conditional flow matching (\Cref{tab:alanine_completion}). With endpoints only, LEAP gives the lowest error among action-based methods ($0.6225$, compared with $0.6609$ for velocity regression and $0.6682$ for binary NRE), while the conditional generators achieve lower reconstruction error. The same LEAP checkpoint is used for both observation patterns.
\subsection{Temperature transfer}
\label{app:alanine_transfer}

We evaluate each learned correction at both temperatures without retraining.
At the target temperature, we use the corresponding reference process and
effective diffusion.

\begin{table}[t]
\centering
\small
\begin{tabular}{@{}ccrrrr@{}}
\toprule
Train & Test & LEAP & Binary NRE & Velocity & Untrained \\
\midrule
$500$ & $500$
& $0.9403{\scriptscriptstyle \pm 0.0099}$
& $0.9337{\scriptscriptstyle \pm 0.0030}$
& $0.9427{\scriptscriptstyle \pm 0.0033}$
& $0.8385{\scriptscriptstyle \pm 0.0050}$ \\
$600$ & $500$
& $\mathbf{0.9579}{\scriptscriptstyle \mathbf{\pm 0.0057}}$
& $0.9489{\scriptscriptstyle \pm 0.0051}$
& $0.9437{\scriptscriptstyle \pm 0.0022}$
& $0.8385{\scriptscriptstyle \pm 0.0050}$ \\
$500$ & $600$
& $\mathbf{0.9719}{\scriptscriptstyle \mathbf{\pm 0.0068}}$
& $0.9386{\scriptscriptstyle \pm 0.0026}$
& $0.9533{\scriptscriptstyle \pm 0.0025}$
& $0.8847{\scriptscriptstyle \pm 0.0028}$ \\
$600$ & $600$
& $0.9754{\scriptscriptstyle \pm 0.0012}$
& $0.9692{\scriptscriptstyle \pm 0.0007}$
& $0.9559{\scriptscriptstyle \pm 0.0029}$
& $0.8847{\scriptscriptstyle \pm 0.0028}$ \\
\bottomrule
\end{tabular}
\caption{\textbf{Temperature transfer.}
Within-context AUROC $\uparrow$. Bold marks LEAP evaluated at a temperature
different from its training temperature.}
\label{tab:alanine_transfer}
\end{table}

Training at $600$\,K gives AUROC $0.9579$ at $500$\,K, while training at
$500$\,K gives $0.9719$ at $600$\,K (\Cref{tab:alanine_transfer}).
Thus the learned correction transfers across temperatures when combined with
the target reference process.

\subsection{Preference and constraint reuse} \label{app:basin_preference}

We keep the learned $600$\,K action fixed and add a periodic preference centered at $c_M=(-60^\circ,-40^\circ)$, \[ M(\Phi) = \frac{1}{N-1} \sum_{k=1}^{N-1} \exp\left[ \frac{1}{s_M^2} \sum_{j=1}^{2} \bigl(\cos(x_{k,j}-c_{M,j})-1\bigr) \right], \qquad s_M=30^\circ, \] with inference objective $J_\lambda(\Phi)=S_\theta(\Phi)-\lambda M(\Phi)$. We evaluate $\lambda\in\{0,5,20,50\}$ on the same $77$ endpoint contexts.

\begin{table}[t]
\centering
\small
\begin{tabular*}{\linewidth}{@{\extracolsep{\fill}}lcccc@{}}
\toprule
& $\lambda=0$ & $\lambda=5$ & $\lambda=20$ & $\lambda=50$ \\
\midrule
Preference $M$ $\uparrow$
& $0.0658{\scriptscriptstyle \pm 0.0059}$
& $0.1095{\scriptscriptstyle \pm 0.0088}$
& $0.2517{\scriptscriptstyle \pm 0.0045}$
& $0.3641{\scriptscriptstyle \pm 0.0116}$ \\
Transition cost $\downarrow$
& $0.9389{\scriptscriptstyle \pm 0.0545}$
& $0.9159{\scriptscriptstyle \pm 0.0228}$
& $0.9809{\scriptscriptstyle \pm 0.0231}$
& $1.3774{\scriptscriptstyle \pm 0.0252}$ \\
\bottomrule
\end{tabular*}
\caption{\textbf{Preference reuse with a fixed LEAP action.}
Increasing $\lambda$ raises the basin preference while changing the unguided projected transition cost in nats per active step.}
\label{tab:basin_preference}
\end{table}

Increasing $\lambda$ from $0$ to $20$ raises the mean preference from $0.0658$ to $0.2517$ while the transition cost changes from $0.9389$ to $0.9809$ (\Cref{tab:basin_preference}). At $\lambda=50$, the preference increases further to $0.3641$ with a larger transition cost of $1.3774$.

\textbf{Constraints and composed requests.} We also keep the learned action fixed while imposing $|\phi_k|\leq150^\circ$ and $|\phi_{k+1}-\phi_k|\leq\pi-10^{-3}$ through projection during inference. We combine these bounds with basin preferences, route constraints and additional observed states by changing only the inference objective or feasible set.

\textbf{Molecular geometry verification.} Reconstructed trajectories preserve molecular geometry at both temperatures, with torsion errors below $0.24^\circ$, relative bond length changes below $1.5\%$, and preserved C$_\alpha$ chirality and trans peptide geometry.

\end{document}